\documentclass[11pt]{article}

\usepackage[utf8]{inputenc}
\usepackage[margin=1in]{geometry}
\usepackage{amsmath, amssymb, amsfonts, amsthm}
\usepackage{mathtools}
\usepackage{graphicx}
\usepackage{booktabs}
\usepackage{multirow}
\usepackage{subcaption}
\usepackage{float}
\usepackage{xcolor}
\usepackage[authoryear,round,sort&compress]{natbib}
\usepackage[ruled,vlined]{algorithm2e}
\definecolor{linkcol}{rgb}{0,0,0.55}
\definecolor{citecol}{rgb}{0,0.45,0}
\definecolor{urlcol}{rgb}{0.55,0,0}
\usepackage[colorlinks=true,linkcolor=linkcol,citecolor=citecol,urlcolor=urlcol]{hyperref}
\PassOptionsToPackage{hyphens}{url}
\usepackage[T1]{fontenc}
\usepackage[scaled=.98,sups]{XCharter}
\usepackage[charter,vvarbb,scaled=1.05]{newtxmath}

\newcommand{\coderelease}{at \url{https://github.com/sushovan4/disentanglement}}

\newenvironment{experiment}[1]{%
  \par\medskip\noindent\textbf{Experiment (#1).}\ \ignorespaces}{%
  \par\medskip}

\newtheorem{theorem}{Theorem}
\newtheorem{proposition}{Proposition}
\newtheorem{lemma}{Lemma}
\newtheorem{corollary}{Corollary}
\newtheorem{definition}{Definition}

\theoremstyle{remark}

\newcommand{\R}{\mathbb{R}}

\newcommand{\E}{\mathbb{E}}

\newcommand{\U}{\mathcal{U}}
\newcommand{\Dx}{\Delta\chi}

\title{Certified Topological Interaction in Neural Representations:\\
Exact Tests and the Statistic They Require}

\author{Sushovan Majhi\thanks{Data Science Program, George Washington University, Washington, D.C., USA. Email: \texttt{s.majhi@gwu.edu}.}}

\date{}

\begin{document}
\maketitle

\begin{abstract}
Class disentanglement---the separation of a representation's class-conditional point clouds along depth and over training---is measured by descriptive curves: the sentence such a study wants to write, \emph{layer $\ell+1$ is more disentangled than layer $\ell$}, is an eyeball judgement with no null. We supply the inferential layer for a topological measurement of class overlap, the Intersection Euler Characteristic Profile: the Euler characteristic of the overlap of the clouds' ball unions as a function of scale, from one Alpha-complex sweep with no boundary-matrix reduction. Every number carries a test---exact permutation tests in both directions, a guarded separation certificate the invariant requires, and a paired sign-flip test for comparative claims.

Building that test taught a lesson outliving this invariant: \emph{its statistic must be scale-free}. On the raw profile mass, which has units of feature length, $12{,}375$ paired tests return $5{,}633$ significant steps of which every one at the first epoch points the wrong way, certifying feature-norm dynamics as disentanglement; the dimensionless statistic returns $2{,}707$, with $2{,}026$ decreases.

Across $111$ networks and $52{,}650$ measurements, disentanglement is depth-graded and early, and interaction quotients rank class pairs by confusability ($\rho=0.83$), on par with cheap separability statistics. In a $96$-model factorial, augmentation is the one training choice that separates classes relative to chance; weight decay compresses the overlap without separating. Only a $k$-fold statistic can pose the structural question: the joint entanglement of a class triple sits below its strongest pair in $97\%$ of triple--layer cells and $99.5\%$ of deep cells, at median ratios far below a measured null floor, in vision encoders and frozen language models---a regularity, not a law. The unnormalized mass predicts test accuracy ($R^2=0.94$), the quotient does not, and neither beats a linear probe.
\end{abstract}

\section{Introduction}\label{sec:intro}

A trained image classifier is, up to its last layer, a map from images to points in a feature space, and the images of two classes---cats and dogs, say---trace out two point clouds there. At the input these clouds are hopelessly intermixed, since a cat and a dog differ mainly in pose and background; at the output they are cleanly apart, since the classifier's last step is a linear cut. Somewhere in between, the network pulls them apart. Figure~\ref{fig:intro_strip} shows this for one pair through the stages of a residual network: at the stem the clouds interpenetrate everywhere, two stages later they are still mixed, and at the third stage they have separated. This paper is about measuring that separation---how much two, or more, classes still overlap at a given layer and at which scales, and whether an observed change between two layers, or two epochs, is more than chance.

\begin{figure}[t]
\centering
\includegraphics[width=\textwidth]{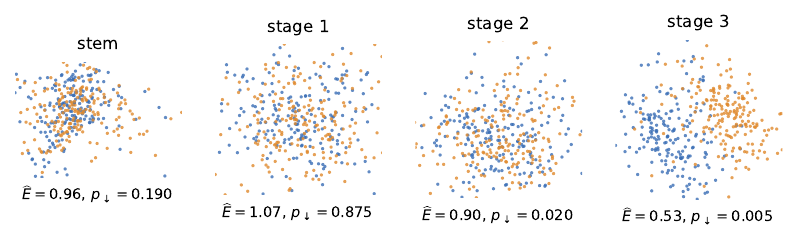}
\caption{What the paper measures. Cat (blue) against dog (orange) through the stages of the CIFAR-10 ResNet-56: a two-dimensional shadow of each stage's $d_0=5$ clouds ($m=200$ per class), with the interaction quotient and separation $p$-value ($B=199$) under each panel. The pair is exchangeable with a random relabeling through stage~1, first certified separated at stage~2, and at quotient $0.53$ at stage~3, the penultimate representation. Definitions in Sections~\ref{sec:background}--\ref{sec:certified}.}
\label{fig:intro_strip}
\end{figure}

\subsection{Class separation, and how it has been measured}\label{sec:separation}

The geometric reading is by now standard. Each class $c$ induces a point cloud $X_c^{(\ell)} = \{f_\ell(x) : \text{label}(x) = c\}$ in the representation at layer $\ell$, and learning can be read as these clouds becoming progressively \emph{disentangled}---topologically separated---with depth and with training. We use ``disentanglement'' throughout in this sense of \emph{class separation} in feature space, following \citet{wagner2024mixup}. It is distinct from the factorized-latent sense of $\beta$-VAE \citep{higgins2017betavae} and of the DCI and mutual-information-gap metrics \citep{eastwood2018framework,chen2018isolating}, in which individual coordinates track independent generative factors; neither sense implies the other, and we make no claim about factorized structure. We do take one lesson from that literature: \citet{locatello2019challenging} showed that its conclusions had been driven by metric choice and random seed, a fact visible only in a sweep over thousands of models, and that descriptive scores without a null distribution have no way to be wrong.

That classes separate has been established repeatedly, in four largely disjoint vocabularies. \emph{Probes}: a linear classifier trained on each layer's features grows more accurate monotonically with depth \citep{alain2017probes}. \emph{Manifold capacity}: in the statistical-mechanics account, per-class object manifolds shrink and decorrelate along the hierarchy and the capacity of a downstream perceptron grows accordingly \citep{cohen2020separability}. \emph{Terminal geometry}: neural collapse \citep{papyan2020prevalence} describes the late-training limit in which class means form a simplex and within-class variability vanishes, the endpoint of disentanglement; its within-class collapse statistic has since been measured layer by layer \citep{rangamani2023feature}, and the generalized discrimination value of \citet{schilling2021quantifying} is a per-layer class-separability score in the same spirit. \emph{Global similarity}: CKA \citep{kornblith2019similarity} compares whole representations across layers and models, and intrinsic dimension falls through the final layers, with the last hidden layer's dimension predicting test accuracy \citep{ansuini2019intrinsic}; neither is class-conditional. Closest to our object, \citet{naitzat2020topology} track the Betti numbers of class regions through depth and show that networks simplify topology layer by layer; their object is the topology of one region at a time, ours the interaction between several.

Each of these is a real measurement, and none of them is a measurement of \emph{overlap}. Probes and capacity are linear-separability notions and answer a question about a decision boundary rather than about the geometry of the clouds: two classes can be perfectly linearly separable while their thickened supports interpenetrate at every scale below the margin, and conversely. CKA and intrinsic dimension are not defined for a specified class subset, so they cannot say \emph{which} classes are still entangled. Collapse statistics and discrimination values are descriptive scores without a null distribution. And none is defined for a set of $k>2$ classes jointly, so none can pose the question of whether three classes remain jointly entangled when no pair does. Quantifying the overlap itself gives a label for representation quality, a diagnostic for training dynamics, and, potentially, a predictor of generalization.

\subsection{Topological interaction, and the inference it lacks}\label{sec:interaction}

\citet{wagner2024mixup} measure pairwise class entanglement with the \emph{mixup barcode},\footnote{The name ``mixup'' is shared with the augmentation of \citet{zhang2018mixup}, with which the barcode has nothing technical in common. The invariant used here is named for the \emph{intersection} of ball unions it measures, following \citet{kawamura2026intersection}.} summarized by a scalar ``total mixup,'' and show that it tracks how a network disentangles its classes through depth and training. The construction has four limitations that matter precisely for this application. It is \emph{pairwise}, defined for an inclusion $L \hookrightarrow K$ of two clouds, and \emph{asymmetric}, since a foreground cloud is privileged; it is \emph{descriptive}, in that no hypothesis test certifies that an observed change exceeds chance; and it is \emph{expensive}, because image-persistence reduction costs $O(N^3)$ in the simplex count, which limits how many layers, epochs, and seeds can be swept. Chromatic topological data analysis \citep{cultrera2024chromatic} also studies several colored clouds, through a six-pack of persistence diagrams built on chromatic Delaunay complexes; it is richer but costlier, likewise untested, and not specialized to representation analysis. Elsewhere, topological data analysis has been applied to network weights \citep{rieck2019neural}, as a differentiable regularizer on latent spaces \citep{moor2020topological} and on class separation at decision boundaries \citep{chen2019topological}, to the persistent homology of labeled complexes that characterize decision boundaries \citep{ramamurthy2019topological}, to neuron correlations as a predictor of test error without a test set \citep{corneanu2020computing}, and to the fractal geometry of optimization trajectories, where persistent-homology dimension predicts generalization \citep{birdal2021intrinsic}. Euler characteristic curves and profiles as a stable, reduction-free shape summary are due to \citet{dlotko2023euler}, are by now an established cheap alternative to persistence in learning pipelines \citep{hacquard2024euler}, and have been made differentiable as transforms \citep{roell2024differentiable}; we apply them to the \emph{interaction} of class clouds rather than to a single dataset.

The third limitation is not peculiar to the mixup barcode. Probe accuracy, manifold capacity, CKA, intrinsic dimension, and collapse statistics are all reported in practice as descriptive numbers, so the sentence a representation-analysis paper actually wants to write---``layer $\ell+1$ is more disentangled than layer $\ell$,'' ``this epoch is where the classes separate''---is usually an eyeball judgment about a curve, with no null distribution behind it and no control of the error rate incurred by reading hundreds of such curves. The inferential tools exist and are textbook: a label-permutation test of exchangeability and a sign-flip test on paired differences \citep{lehmann2005testing}. Permutation-calibrated two-sample tests of the same exchangeable null are standard in machine learning---kernel MMD \citep{gretton2012kernel}, energy distance \citep{szekely2013energy}, classifier two-sample tests \citep{lopezpaz2017revisiting}---and any of the descriptive separability scores above can be calibrated the same way, as we do for four of them (Experiment~E3). In topological data analysis, confidence sets for persistence diagrams separate signal from noise in a single sample \citep{fasy2014confidence}, permutation tests for collections of diagrams are due to \citet{robinson2017hypothesis}, topological goodness-of-fit tests to \citet{dlotko2022topotests}, a universal null distribution for persistence-based statistics to \citet{bobrowski2023universal}, and minimax hypothesis tests and confidence sets for populations of diagrams, through their landmark embeddings, to \citet{bagchi2026inference}. The object tested there is the shape of one sample or a comparison of diagrams, not the interaction of two labeled samples, and we are not aware of a paired test for topological summaries. What is missing, then, is not an inferential layer for the field but the routine application of these tools to comparative claims about representations. Doing that for a topological measure of class overlap---together with the guard and the scale-free paired statistic that this particular measure turns out to need---is the main business of this paper.

\subsection{Our contributions}\label{sec:thispaper}

We measure class disentanglement with the \emph{Intersection Euler Characteristic Profile} (Intersection ECP) of \citet{kawamura2026intersection}: the Euler characteristic of the region where the thickened class clouds overlap, as a function of the thickening scale (Section~\ref{sec:background}). It removes each of the four limitations. It is symmetric; it is defined for every $k$, so that it measures joint, multi-class entanglement and not only pairs; it is computed by a single sorted Alpha-complex sweep with no matrix reduction; and it admits an exact, distribution-free permutation test. Section~\ref{sec:certified} turns it into a measurement protocol with a certificate attached to every number: exact one-sided tests in both directions, the \emph{guarded} separation certificate that the Euler characteristic specifically needs, a paired sign-flip test on a scale-free statistic for the comparative claims that trajectory analyses actually make, and an interaction quotient comparable across layers, epochs and models. The guarantees are elementary and are proved in Appendix~\ref{app:inference}. ``Certified'' here means a hypothesis test with an exact level, not the margin certificate that landmark-based persistence vectorizations now carry \citep{majhi2026place,majhi2026palace}.

One requirement of that protocol is not a technicality and is the paper's most portable result. A paired test compares a statistic across two conditions, so the statistic must be \emph{scale-free}; the natural one here, the profile mass, has units of feature length. Paired on the raw mass, our first campaign certified $5{,}633$ epoch-to-epoch steps and pointed the wrong way at every significant step of the first epoch, because the permutation-null mean of the mass grows twentyfold there as the features spread out. The dimensionless statistic certifies $2{,}707$ steps and points the right way. The absolute certificates, which divide by a null computed on the same grid, were never affected; only the comparative ones were. Any trajectory analysis resting on an unnormalized geometric statistic---and most representation-analysis statistics are unnormalized---meets this failure mode, and Experiment~E3 reports ours in full.

The campaign of Section~\ref{sec:experiments} covers $111$ trained networks and $52{,}650$ certified measurements. Its findings are four.
\begin{itemize}
  \item \emph{Disentanglement is certified, depth-graded and early.} The interaction quotient ranks MNIST pairs by confusability (Spearman $\rho=0.83$; E1); deep layers do almost all of the separating (E2), and most certified changes fall within the first eight epochs (E3); a from-scratch ViT disentangles non-monotonically, an artifact that pretraining removes (E2). The tests are calibrated (E4), and every conclusion survives the choice of projection dimension (E8).
  \item \emph{Pairwise dominance}, the finding only a $k$-fold statistic can pose. The joint entanglement of a class triple sits below that of its strongest pair in $97\%$ of triple--layer cells and $99.5\%$ of deep cells---a rate that means little by itself, since the score divides one quotient by a maximum of three and an unstructured family scores $83\%$---and does so at median ratios far below that measured floor, with the exceptions listed (E2) and the interaction spectrum decaying in order (E7). It is a regularity, not a law: expected from the nesting of overlaps but not forced by geometry, present at initialization and in raw pixels, manufactured in the last stage alone when a network memorizes random labels (E9), and reproduced on two frozen language models (E10).
  \item \emph{What the measure is for.} In a $96$-model factorial population, augmentation is the one training choice that separates classes relative to chance; weight decay compresses the overlap without separating; depth and width do nothing (E11). The unnormalized profile mass predicts test accuracy ($R^2=0.94$), the quotient does not, and neither beats a linear probe (E5).
  \item \emph{A differentiable surrogate steers the wrong way}, in two implementations and each for an identifiable reason (E6), which is itself an argument for keeping the certificate non-differentiable.
\end{itemize}
Following \citet{turkes2022effectiveness}, we ask on which tasks the topological statistic beats the alternatives and say where it does not: for ranking pairwise confusability and for timing disentanglement over training, a nearest-neighbour cross-class rate on the same clouds does as well (E1, E3). Its value lies in the certificate, the $k$-fold terms, and the scale resolution, which the cheap statistics do not offer.

\section{The Intersection ECP}\label{sec:background}

We recall only what is needed; all constructions, proofs, stability, and complexity are in \citet{kawamura2026intersection}.

\begin{figure}[t]
\centering
\includegraphics[width=\textwidth]{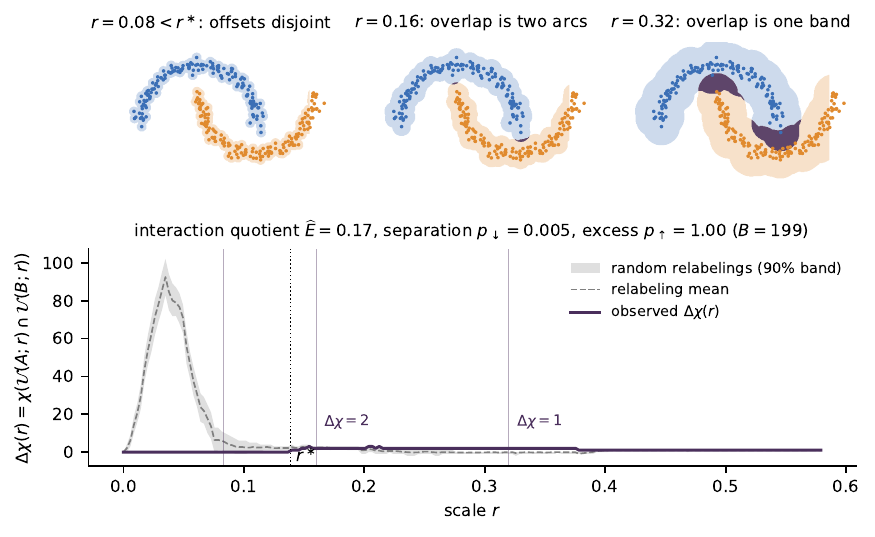}
\caption{The measurement, on two interlocking clouds. \emph{Top:} the offsets at three scales with the intersection shaded---disjoint below $r^\ast$, then two arcs ($\chi=2$), then one band ($\chi=1$). \emph{Bottom:} the profile $\Dx(r)$ against random relabelings of the pooled points (mean and $90\%$ band): quotient $0.17$, separation certified at $p_\downarrow=0.005$, no excess.}
\label{fig:intro_profile}
\end{figure}

\subsection{The invariant and its properties}\label{sec:invariant}

Write $\U(S;r)=\bigcup_{p\in S}B(p,r)$ for the $r$-offset of a set $S\subseteq\R^d$; following \citet{kawamura2026intersection} we use the same notation whether $S$ is a finite cloud or a compact support, so that sample and population offsets read alike. For finite clouds $X_1,\dots,X_k\subset\R^d$ and scales $\mathbf t=(t_1,\dots,t_k)$, the Intersection ECP is
\[
    \Dx(\mathbf t;\, X_1,\dots,X_k) \;=\; \chi\Big(\textstyle\bigcap_{i=1}^k \U(X_i; t_i)\Big),
\]
the Euler characteristic of the overlap of the thickened clouds. On the diagonal $\mathbf t=(r,\dots,r)$ we write $\Dx(r)$; for $k=2$ it is the inclusion--exclusion defect $\chi_X(r)+\chi_Y(r)-\chi_{X\cup Y}(r)$. It is a piecewise-constant integer curve in $r$, identically zero below first contact \citep[Lem.~3.1]{kawamura2026intersection}, equal to the number of contact points at first contact \citep[Lem.~3.2]{kawamura2026intersection}, and thereafter zero exactly when the overlap's Betti numbers cancel (Section~\ref{sec:blind}). Figure~\ref{fig:intro_profile} shows the construction and the test on two clouds in the plane.

Four properties carry the paper. Three are established in \citet{kawamura2026intersection}; the remaining one, the inferential layer (P3), is what the present paper supplies---\citeauthor{kawamura2026intersection} develop the invariant, its stability, and its sampling theory, and leave hypothesis testing to future work.
\begin{enumerate}
  \item[\textbf{P1}] \emph{Symmetry and $k$-naturality.} $\Dx$ treats the clouds symmetrically and is defined for every $k$ by intersecting $k$ ball unions; no foreground class, no choice of order.
  \item[\textbf{P2}] \emph{Cheap computation.} $\Dx(\cdot;X,Y)$ is computed from the Euler characteristic curves of the Alpha complexes of $X$, $Y$, $X\cup Y$ by a single sorted sweep of signed simplex counts, in $O(n^{\lceil d/2\rceil}\log n)$ time with no boundary-matrix reduction \citep[Thm.~6.5]{kawamura2026intersection}---asymptotically below the $O(N^3)$ of image persistence.
  \item[\textbf{P3}] \emph{Hypothesis test.} An exact, distribution-free permutation test (label shuffling) certifies at level $\alpha$ that two clouds interact more---or, in the opposite tail, \emph{less}---than chance, for any sample size. This is developed in Section~\ref{sec:tests}.
  \item[\textbf{P4}] \emph{Stability.} under a Hausdorff perturbation of size $\epsilon$ the intersection filtrations are $\epsilon$-interleaved, their diagrams are within bottleneck distance $\epsilon$, and the profiles satisfy $d_{L^1}(\Dx,\Dx')\le2\epsilon(N+N')$ with $N,N'$ the bar counts \citep[Thm.~2.12, Prop.~3.12]{kawamura2026intersection}; the finite-sample guarantee proper is the recovery theorem (Section~\ref{sec:estimation}).
\end{enumerate}

\subsection{What $\Dx$ cannot see}\label{sec:blind}

\begin{figure}[t]
\centering
\includegraphics[width=\textwidth]{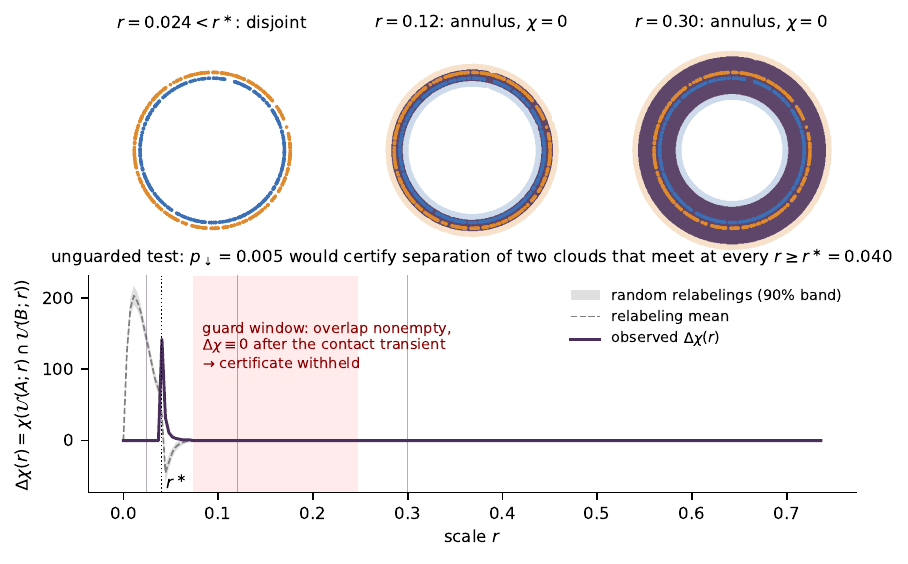}
\caption{The blind spot, and the guard. Two concentric rings meet from $r^\ast=0.04$ on, their overlap is an annulus, and $\chi$ vanishes on it at every larger scale, so the unguarded test would certify separation ($p_\downarrow=0.005$; Theorem~\ref{thm:false}). The guard sees a nonempty overlap with $\Dx\equiv0$ after the contact transient (shaded) and withholds the certificate.}
\label{fig:intro_blindspot}
\end{figure}

Because $\chi$ is an alternating sum, an overlap region with balanced Betti numbers is invisible: an annulus- or torus-shell-shaped intersection has $\chi=0$ across an entire scale plateau, bit-identical to \emph{no} overlap at all---a structural, scale-stable cancellation that the profile provably cannot detect \citep[\S4 and the saturation experiment therein]{kawamura2026intersection}. For disentanglement measurement this is a live failure mode, not a corner case: two class clouds arranged around a shared loop would register as perfectly separated (Figure~\ref{fig:intro_blindspot}). The quantity that disambiguates is the \emph{first interaction scale}
\[
    r^\ast \;=\; \inf\{r \ge 0 : \U(X;r)\cap\U(Y;r)\neq\emptyset\},
\]
which is half the smallest cross-class distance, the diagonal boundary of the dead zone of \citet[Lem.~3.2]{kawamura2026intersection}; it falls out of the same Alpha-complex sweep at no extra cost and separates ``zero because empty'' from ``zero because cancelling.'' Section~\ref{sec:tests} builds it into the separation certificate. When $\Dx\approx0$ on a range where $r^\ast$ certifies a nonempty overlap, the escalation is to the Betti curves of the overlap filtration---the overlap-homology comparison of their saturation experiment---or to the relative-homology refinement of \citet[\S4]{kawamura2026intersection} itself; either separates the two cases outright, at the cost of a persistence computation on the overlap alone.

\section{Certified measurement of disentanglement}\label{sec:certified}

Every number reported below carries a test. This section fixes the measurement---class clouds, projection, grid, and the interaction quotient---and then the tests: what each certifies, in which direction, and how multiplicity is controlled. The guarantees are elementary and are proved in Appendix~\ref{app:inference}; sharper procedures---structured multiplicity control along the spectrum's floors, and an analytic null that would calibrate without permutations---are left to future work.

\subsection{Class clouds, projection, and the interaction quotient}\label{sec:setup}

Fix a trained (or mid-training) network and a layer $\ell$ with feature map $f_\ell$. For classes $c_1,\dots,c_k$ draw $m$ examples each and form $X_i^{(\ell)}=\{f_\ell(x):\text{label}(x)=c_i\}$. Because the Alpha complex scales as $n^{\lceil d/2\rceil}$, we project to a low dimension $d_0$ by PCA before computing (E8 and Appendix~\ref{app:e8} study sensitivity to $d_0$); this matches the protocol of \citet{wagner2024mixup}. The projection is fitted once per layer on the \emph{pooled} features of all classes under test (the ten classes of a network; the pair itself in E1) and never sees the labels, and the scale grid below is likewise a function of the pooled cloud, so both are invariant under the relabelings of Section~\ref{sec:tests} and the permutation tests remain exact. Concretely, features are those of the held-out images (the test split; for the CIFAR-100 subset, whose test split has only $100$ images per class, the unaugmented training split of $500$ per class), the headline clouds are the first $m=200$ images of each class, and the paired tests of Section~\ref{sec:paired} use $F=8$ disjoint folds of $100$ images per class ($62$ for the CIFAR-100 subset) cut from the same projection. The scale grid has $200$ points from $0$ to half the median pairwise distance of the pooled pair, and every functional below is evaluated on it. The grid is a convenience, not a necessity: the profile is a step function whose breakpoints are the critical values of the three Alpha filtrations (some $7\times10^4$ per pair at $m=200$, $d_0=5$), so its mass could be summed exactly; on the $45$ MNIST pairs of E1 the grid mass is within $\pm2\%$ of the exact one, and the permutation tests are exact for either statistic. What the grid does fix is the range of scales, its upper end. Projection is also conservative in the direction we certify:

\begin{proposition}[Projection is conservative for separation]\label{prop:projection}
Let $P:\R^D\to\R^{d_0}$ be an orthogonal projection and let $X,Y\subset\R^D$ be finite. Writing $r^\ast(X,Y)=\tfrac12\min_{x\in X,y\in Y}\|x-y\|$ for the first interaction scale of Section~\ref{sec:blind},
\[
    r^\ast(PX,PY)\;\le\;r^\ast(X,Y).
\]
Consequently, if the projected offsets are disjoint at scale $r$---that is, $r<r^\ast(PX,PY)$---then $\U(X;r)\cap\U(Y;r)=\emptyset$ in the ambient representation as well.
\end{proposition}
\begin{proof}
$P$ has operator norm $1$, so $\|Px-Py\|=\|P(x-y)\|\le\|x-y\|$ for all $x,y$; taking the minimum over $x\in X$, $y\in Y$ gives the inequality. If $r<r^\ast(PX,PY)$ then $r<r^\ast(X,Y)$, and $\U(X;r)\cap\U(Y;r)\neq\emptyset$ would force some $\|x-y\|<2r$, a contradiction.
\end{proof}

Projection can therefore only make classes look \emph{more} entangled, never more separated: every first-contact scale, and hence every guard decision of the form $\hat r^\ast>r$, transfers from the $d_0$-dimensional shadow to the full-dimensional representation; the permutation certificate itself is a statement about the projected clouds. The converse fails---apparent entanglement in the projection may be an artifact of it---which is one more reason the claims we certify are the separation claims (Section~\ref{sec:tests}), and why ``remains entangled'' is only ever a magnitude comparison here.

The \emph{entanglement} of classes $(c_1,\dots,c_k)$ at layer $\ell$ is a functional of the Intersection ECP of their class clouds,
\[
    E_\ell(c_1,\dots,c_k) \;=\; \Phi\big(\Dx(\cdot;\, X_1^{(\ell)},\dots,X_k^{(\ell)})\big), \qquad \Phi \in \Big\{\, \max_r|\Dx(r)|,\ \textstyle\int|\Dx(r)|\,dr \,\Big\}.
\]
Both functionals vanish for separated clouds and grow with entanglement. The $L^1$ bound of P4 controls the change of the integrated mass, which is the functional used for every certified claim below; it does not control the peak statistic \citep[\S3.5]{kawamura2026intersection}, which we report only descriptively. Raw values of $\Phi$ carry the layer's density, dimension and feature scale, so we normalize to the \emph{interaction quotient}
\[
    \widehat E_\ell \;=\; \frac{\Phi\big(\Dx(\cdot;\,X_1^{(\ell)},\dots,X_k^{(\ell)})\big)}{\E_{H_0}\big[\Phi\big]},
\]
where the denominator is the expected value of the same functional under the exchangeable null---estimated at no extra cost as the mean of the $B$ permutation replicates already computed for the test of Section~\ref{sec:tests}. The quotient is dimensionless, is centred at $1$ when the classes are exchangeable (numerator and denominator are then evaluations of one statistic on exchangeable splits), and tends to $0$ under separation; because numerator and denominator share the layer's density, dimension $d_0$, sample size and grid, it is comparable \emph{across} layers, epochs, and models, which raw $\Dx$ values are not. \emph{Disentanglement} is a \emph{decrease} of $\widehat E_\ell$ along depth ($\ell$ increasing) or training (epoch increasing). Throughout, the head statistic is the raw mass $\int|\Dx|\,dr$ and its quotient; the paired tests use a dimensionless form of the mass (Section~\ref{sec:paired}).

\subsection{Two one-sided tests, and the guarded certificate}\label{sec:tests}

Let $H_0$ be
that $X$ and $Y$ are i.i.d.\ from a common distribution. Under $H_0$ the labels
are exchangeable, so for a profile functional $T=\Phi(\Dx)$ and $B$ random
relabelings,
\[
  p_{\uparrow} = \frac{1+\#\{b: T^{\sigma_b}\ge T_{\mathrm{obs}}\}}{1+B},
  \qquad
  p_{\downarrow} = \frac{1+\#\{b: T^{\sigma_b}\le T_{\mathrm{obs}}\}}{1+B}
\]
are both exact at any sample size and any $B$ (Proposition~\ref{prop:exact}). The distinction matters and conflating
the two costs all power against one of them: $p_\uparrow$ detects \emph{excess}
interaction, while $p_\downarrow$ detects \emph{separation}---the overlap
topology is poorer than a random relabeling predicts. For class clouds the
permuted splits interpenetrate maximally, so a layer that separates its classes
drives $T_{\mathrm{obs}}$ into the \emph{lower} tail. It is $p_\downarrow$, not
$p_\uparrow$, that certifies disentanglement. The low-tail permutation test on this invariant was introduced for regime detection in dynamical systems by \citet{majhi2026regime}, where a transition is declared when the overlap of two time-series windows is unusually low; here the same test certifies the separation of class clouds, and the paired form of Section~\ref{sec:paired} is new. The right tail is retained as a directional sanity check but, empirically, never fires on labeled class clouds---no class split interpenetrates \emph{more} than a random relabeling (E1: $0/45$ MNIST pairs)---so claims below that a pair ``remains entangled'' are statements about the magnitude of its interaction quotient relative to other pairs and conditions, not about $p_\uparrow$.

The separation certificate must be guarded. On a configuration that is $\chi$-blind in the sense of Section~\ref{sec:blind}, the separation test does not merely lose power: it rejects with probability tending to one, certifying separation exactly where the clouds meet (Theorem~\ref{thm:false}). The repair is the first interaction scale $r^\ast$, which falls out of the same sweep at no cost and is a consistent estimate of the population first-contact scale (Lemma~\ref{lem:rstar}): the \emph{guarded certificate} withholds a separation claim whenever the profile is silent on the scales immediately beyond $\hat r^\ast$, where the sampled offsets provably meet (Corollary~\ref{prop:guard}). Every separation claim in this paper is
of the guarded kind, and Experiment~E4 reports how often the guard fired. Operationally, a pair is \emph{certified separated} in a cell when its $p_\downarrow$ survives Benjamini--Hochberg within that cell's family of $45$ pairs and the guard is silent, the guard being evaluated on the quarter of the grid immediately beyond $\hat r^\ast$. A sampled pair first meets in isolated contact points, a transient of a few grid points with $\Dx>0$, so the released code also implements a \emph{plateau} form of the guard: it waits for the profile's first return to zero after contact and fires only if that return comes within the quarter-window and the profile then stays at zero on at least $90\%$ of the following quarter, the signature of an annular overlap and not of a blob (Figure~\ref{fig:intro_blindspot}). Both forms were evaluated on the campaign (E4). The certificate says that the overlap topology of the pair is significantly poorer, over the whole scale range, than a random relabeling produces; it does \emph{not} say the offsets are disjoint. What remains of the interaction is what the quotient reports, and certified pairs carry quotients anywhere from $0.05$ to $0.9$.

\subsection{Comparative claims: the paired test}\label{sec:paired}

Rejecting $H_0$ is rarely the interesting
event here: distinct classes are essentially never exchangeable at \emph{any}
layer of a trained network, so the one-sided tests answer only the absolute
question. The claims that carry the application are comparative---\emph{layer
$\ell+1$ is more disentangled than layer $\ell$}, \emph{entanglement dropped
between epochs $e$ and $e'$}---and need a two-condition inference. We use a paired subsampled design (Appendix~\ref{app:paired}): split the inputs into $F$
disjoint folds, pass each fold through both conditions to obtain paired
replicates and their differences $D_s$, and calibrate by the \emph{sign-flip}
test, which is exact at finite $F$ for the symmetric null; its smallest attainable two-sided $p$-value is $2/2^{F}$, so it can reject at level $\alpha$ whenever $2^{F-1}\ge1/\alpha$ ($F\ge6$ at $\alpha=0.05$). We take $F=8$, so the floor is $2/256\approx0.008$. The sign of $\bar D$ is read off after the two-sided test. The fold statistic must be \emph{scale-free}. The raw profile mass $\int|\Dx|\,dr$ has units of feature length on the data-adaptive grid of Section~\ref{sec:setup}, so a paired test on it across epochs certifies changes of feature \emph{norm} as changes of interaction; we therefore pair the dimensionless mass $\int|\Dx|\,dr/r_{\max}$, with $r_{\max}$ the top of the grid (dividing by the pair's permutation-null mean at that checkpoint instead gives the same picture). Experiment~E3 reports what the raw statistic certified instead, as a cautionary contrast.

\subsection{Estimation and multiplicity}\label{sec:estimation}

At scales regular for both the population and the sample filtration, the sample profile equals the population one once the samples are $\epsilon$-dense, with a sample complexity of order $\epsilon^{-k}(\log(1/\epsilon)+\log(1/\delta))$ in the \emph{intrinsic} dimension $k$ of the class supports rather than the ambient or projected feature dimension (Proposition~\ref{prop:rate}, restated from the foundations paper). This is what licenses reading the
measured profile as a property of the representation rather than of the sample.

Across the layers $\times$ pairs $\times$ epochs sweep we control the false discovery rate by Benjamini--Hochberg at $q=0.05$, applied separately within each claim family; for the trajectory tests of E3 a family is the $45$ class pairs of one seed, layer, and epoch step. Those $45$ statistics share the fold split and the pooled clouds, so they are dependent; we assume they satisfy positive regression dependence on the true nulls, under which Benjamini--Hochberg controls the false discovery rate at $q$ \citep[Thm.~1.2]{benjamini2001control}, we do not prove it, and we check its cost with the dependence-agnostic correction \citep[Thm.~1.3]{benjamini2001control}: the dependence-agnostic Benjamini--Yekutieli correction has threshold $iq/(45H_{45})$ at rank $i$, with $H_{45}\approx4.40$, so at $F=8$ it can reject only when at least $31$ of the $45$ paired $p$-values sit at the floor $2/256$. Run over the E3 families it retains $254$ of the $2{,}707$ BH rejections, all of them at the first epoch step: the early-training conclusion of E3 survives the stricter correction, the late learning-rate wave does not, and we say so where each is reported. Procedures that exploit the structure BH ignores (the spectrum's floors form a descending chain, so their occupancy hypotheses admit fixed-sequence testing) need permutation replicates aligned across the family, which this campaign did not retain.

\subsection{Multi-class structure}\label{sec:multiclass}

Unlike total mixup, $E_\ell$ is defined for $k>2$, in three informative regimes: (i) the \emph{pairwise matrix} $E_\ell(c,c')$ over all class pairs, a representation-geometry analogue of a confusion matrix; (ii) the \emph{higher-order} terms $E_\ell(c,c',c'')$ and beyond, which detect class triples (or larger groups) that remain \emph{jointly} entangled when no pair is; and (iii) the \emph{interaction spectrum} \citep[Def.~2.19]{kawamura2026intersection}, the floors $\chi(C_j)$, $C_j(r)=\bigcup_{|S|=j}\bigcap_{i\in S}\U(X_i;r)$ (``at least $j$ classes meet''), of which $\Dx$ is the top floor, exactly computable from pooled Euler curves by Jordan's sieve \citep[Ch.~IV]{comtet1974advanced} applied to the Euler integral of \citet[Def.~2.19]{kawamura2026intersection}, $\chi(C_j)=\sum_{m\ge j}(-1)^{m-j}\binom{m-1}{j-1}\sum_{|S|=m}\Dx_S$ (the case $j=1$ is their Lemma~A.1). Each floor inherits the permutation tests and the interaction quotient verbatim, with the $k$-cloud null of Proposition~\ref{prop:exact}. Since every floor is \emph{linear} in the subset profiles, the spectrum adds no information beyond the subset family; its value is as a graded summary---the decay of $\widehat E(C_j)$ in $j$ is the quantitative form of pairwise dominance, and the onset delays $r^\ast(C_j)-r^\ast(C_2)$ record how much later deep multi-class overlap begins than first pairwise contact (Experiments~E2, E7). The second floor is also the natural \emph{single} summary of a layer: one number, one test, no multiplicity across pairs. On the five-class families of E7 it tracks the mean pairwise quotient of the same ten pairs at every layer of every model, a little above it because the union of overlaps is normalized by its own null (Appendix~\ref{app:e7}); we report the pairwise mean in the experiments because the questions asked there are about particular pairs. Algorithm~\ref{alg:disentangle} summarizes the protocol.

\begin{algorithm}[t]
\caption{Layerwise disentanglement profile}\label{alg:disentangle}
\KwIn{network $f$; classes $c_1,\dots,c_k$; $m$ samples/class; PCA dim $d_0$; functional $\Phi$}
\KwOut{$E_\ell$ and $p$-value for each layer $\ell$ (and class subset)}
\ForEach{layer $\ell$}{
  collect class clouds $X_i^{(\ell)}=\{f_\ell(x):\text{label}(x)=c_i\}$; PCA to $\R^{d_0}$\;
  compute $\Dx(\cdot)$ by Algorithm~1 of \citet{kawamura2026intersection} (Alpha-complex sweep, no reduction)\;
  $E_\ell \gets \Phi(\Dx)$\;
  permutation test: shuffle labels $B$ times, recompute $\Phi$, report $p$-value (P3)\;
}
\Return $\{(E_\ell, p_\ell)\}_\ell$ and the multi-class tensor $\{E_\ell(c_S)\}$\;
\end{algorithm}

\section{Experiments}\label{sec:experiments}
The campaign comprises $111$ trained networks ($104$ in the main suite: the five E3 seeds, the three E2 models and the $96$-model population of E5; plus a follow-up suite: two SimCLR encoders, a pretrained ViT-B/16 finetune, and a $4$-model learning-rate ablation), $1{,}170$ (network, layer, checkpoint) cells, and $52{,}650$ certified pairwise measurements---$42{,}570$ distinct clouds, since the stage-3 and penultimate rows of a ResNet are the same clouds measured twice---(the control arms of E6 and E9 add $17$ further trained networks, and the initialization and language-model cells of E9--E10 lie outside these counts), computed by the GUDHI implementation of the Alpha-complex sweep \citep{kawamura2026intersection} on CPU nodes with no persistence reduction anywhere. Protocol constants throughout: $m=200$ points/class, $d_0=5$ (PCA), $B=199$ permutations per test ($B=999$ for escalations), $F=8$ folds for paired tests, BH at $q=0.05$ within claim families. Datasets are MNIST \citep{lecun1998gradient} and CIFAR-10/100 \citep{krizhevsky2009learning}; architectures are MLPs, residual networks \citep{he2016deep}, vision transformers \citep{dosovitskiy2021image}, and, for the label-free regime, SimCLR encoders \citep{chen2020simple}. Training recipes: CIFAR ResNets use SGD with momentum $0.9$, learning rate $0.1$, batch size $128$, cosine decay over $80$ epochs, weight decay $5\cdot10^{-4}$ and random-crop/flip augmentation unless an experiment says otherwise (the E5 population crosses weight decay $\{0, 5\cdot10^{-4}\}$ with augmentation on/off); the from-scratch ViT uses AdamW at learning rate $10^{-3}$, weight decay $0.05$, $100$ epochs; the ViT-B/16 finetune AdamW at $10^{-4}$, weight decay $10^{-4}$, $8$ epochs, batch size $64$; SimCLR learning rate $0.5$, weight decay $10^{-4}$, $200$ epochs, batch size $256$; the random-label control weight decay $0$, no augmentation, $150$ epochs. E5's regressions are ridge with the penalty chosen by cross-validation over $10^{-3}$--$10^{3}$. Code, configuration files, every measurement record, and the scripts that regenerate each number and figure accompany the paper.

\begin{figure}[t]
\centering
\includegraphics[width=0.55\textwidth]{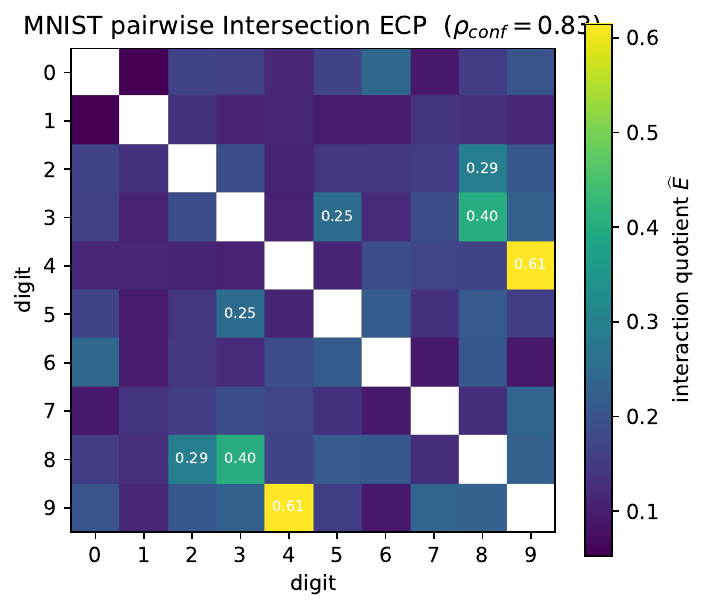}
\caption{Pairwise interaction quotients over the ten MNIST digits ($m=200$/class, $d_0=5$). All $45$ pairs are certified separated; the residual quotient ranks pairs by confusability (Spearman $\rho=0.83$ with a classifier's confusion matrix). Annotated cells are the most entangled pairs.}
\label{fig:mnist_heatmap}
\end{figure}

\begin{experiment}{E1 --- reproduce Wagner, at scale}
Correlate the pairwise matrix of Figure~\ref{fig:mnist_heatmap} $\widehat E(c,c')$ (interaction quotient, $\Phi=\int|\Dx|$) entrywise with the confusion matrix of a trained classifier (Spearman rank correlation over the 45 pairs), and verify the ranking is stable across $d_0\in\{3,4,5,6\}$, points per class $m\in\{50,100,200,400\}$, and 5 independent resamples ($d_0>6$ is out of reach for Alpha complexes: a 7-dimensional complex on $400$ points already exceeds $10^7$ simplices).

The representation is the raw $28\times28$ pixel vector in $[0,1]$ (no network), and the clouds are $m=200$ training images per class projected by a PCA fitted on the pooled pair; the confusion matrix is that of a multinomial logistic classifier on the top $50$ pixel principal components, evaluated on the MNIST test set and symmetrized. On the identical $d_0=5$ clouds we also compute four cheap separability statistics---the Fisher ratio $\|\mu_1-\mu_2\|^2$ over the mean within-class variance, the energy distance, the $5$-nearest-neighbour cross-class rate, and the cross-validated error of a logistic probe---and rank the pairs by each.

\emph{Result} ($m=200$, $d_0=5$, $B=199$). All $45$ pairs are certified separated at $p_\downarrow\le0.05$, and the residual interaction quotient ranks pairs by confusability: Spearman $\rho=0.83$ against the classifier's symmetrized confusion matrix ($p=2.4\times10^{-12}$). The cheap statistics rank it comparably: $5$-NN cross-class rate $\rho=0.86$, probe error $0.83$, energy distance $0.77$, Fisher ratio $0.74$; the quotient correlates $0.87$ with the neighbour rate itself. For ranking pairwise confusability, then, the topological statistic adds nothing over a nearest-neighbour count, and we do not claim otherwise; its value lies in the certificate, the $k$-fold terms and the scale resolution that the cheap statistics do not offer. The most entangled pairs are $(4,9)$ ($\widehat E=0.61$), $(3,8)$ ($0.40$), $(2,8)$ ($0.29$); the least, $(0,1)$ ($0.05$), $(6,9)$ and $(0,7)$ ($0.09$). The ranking is stable across the full protocol sweep (Spearman vs.\ reference: median $0.87$, minimum $0.73$ over $80$ sweep cells). No pair triggers $p_\uparrow\le0.05$, confirming that the certifiable direction on class data is separation (Section~\ref{sec:tests}). See Figure~\ref{fig:mnist_heatmap}.
\end{experiment}

\begin{experiment}{E2 --- multi-class disentanglement vs depth}
On CIFAR-10 (ResNet-56 and ViT-Tiny) and a 10-class subset of CIFAR-100 (ResNet-56), compute $\widehat E_\ell$ through depth at every residual stage / transformer block: the full pairwise matrix, and a scan over all $\binom{10}{3}=120$ class triples reporting the joint statistic $\widehat E_\ell(c,c',c'')$ against the maximum of its three pairwise restrictions (the $k=3$ inclusion--exclusion has $7$ Alpha-complex terms per evaluation, so the full-depth triple scan runs at $d_0=4$, with the deep layers rescanned at $d_0=5$; E8 covers the $d_0$ sensitivity of the pairwise results). One bookkeeping note applies to every ResNet cell in the paper: the penultimate representation is the globally pooled output of the last residual stage, so the ``stage~3'' and ``penultimate'' rows of a ResNet are the same point clouds measured twice, with independent permutation draws. Counts of measurements below include both rows; where a count is itself a claim (the dominance rate here) we report the deduplicated number.

\emph{Result: depth.} Disentanglement is certified layer-by-layer and is strongly depth-graded (Figure~\ref{fig:e2_depth}). The CIFAR-10 ResNet-56's mean pairwise quotient falls $0.72 \to 0.63 \to 0.57 \to 0.18$ from stem to stage~3 ($45/45$ pairs certified at the deep stages); on the CIFAR-100 subset the deep-layer drop is sharper still ($0.79 \to 0.06$). The quotient tracks linear-probe accuracy inversely at every layer. The ViT is a contrast of standard recipes as much as of architectures (it is trained with AdamW at weight decay $0.05$, the ResNets with SGD at $5\cdot10^{-4}$): it disentangles less at every depth, and \emph{non-monotonically}---its mean quotient falls to $0.31$ at block~6 and rises again to $0.39$ at the penultimate representation, with individual pairs remaining at quotients up to $0.88$; consistently, it is the weakest generalizer of the three models. The cheap statistics of E1 computed on the same clouds give the same depth ordering for the ResNets---on CIFAR-10 the $5$-NN cross-class rate falls $0.36\to0.03$ and the Fisher ratio rises $0.5\to10$ from stem to stage~3---and disagree with the quotient in exactly one place, this ViT's head, where they keep improving from block~6 to the penultimate representation ($0.13\to0.09$; $2.6\to4.2$) while the quotient rises $0.31\to0.39$: the overlap topology re-entangles even as nearest-neighbour purity and mean separation improve. The non-monotonicity is a \emph{from-scratch artifact}, not an architectural signature: finetuning a pretrained ViT-B/16 on the same data yields a monotone depth profile at every finetuning epoch, up to $\pm0.01$ permutation noise among the deepest blocks---mean quotient $0.68 \to 0.53 \to 0.14 \to 0.13 \to 0.13$ from block~3 to the penultimate representation at epoch~6, with $45/45$ pairs certified from block~9 down---and the profile is already monotone at epoch~0, i.e.\ before any finetuning. Sufficient (pre-)training data, not the attention architecture, determines whether a ViT disentangles monotonically.

\emph{Result: higher-order scan.} Across all $120$ triples at every layer of the three architectures ($d_0=4$; $1{,}560$ distinct triple--layer cells after removing the stage-3/penultimate duplicates), the joint quotient sits at or below the largest of its three pairwise quotients in $1{,}509$ cells ($96.7\%$), at an overall median ratio $\widehat E_{\mathrm{triple}}/\max_{\mathrm{pairs}}\widehat E$ of $0.63$ (Figure~\ref{fig:e2_triple}). The exceptions are reported in full: $44$ on the CIFAR-100-subset ResNet-56, spread over all four stages (ratios up to $1.11$ in the stem and $1.22$ in stage~3, on a model whose early-layer pairs sit near the null), and $7$ on the ViT---six in block~6, the block at which its depth profile turns, the largest at $1.41$, and one in block~2 at $1.04$; the CIFAR-10 ResNet-56 has none. A dimension-matched rescan of the deep layers at $d_0=5$ ($600$ distinct cells: the deep stage of each ResNet-56 and ViT blocks~6, 8 and the penultimate representation; also rerun at mid-training checkpoints of the E3 ResNet-20 with the same outcome) gives $597/600$ ($99.5\%$), median $0.47$, maximum $1.07$, the three exceedances all on the CIFAR-100 subset. The pattern that would falsify pairwise dominance outright---all three pairs certified separated with quotients $\le 0.15$ while the joint quotient is $\ge 0.2$ and at least double the pairwise maximum---returns \emph{zero} candidates at every layer, under both the screening and the escalated $B=999$ tests. Dominance carries a test against the null floor measured in E9: an unstructured family registers as dominated $66$ times in $80$, with per-seed median ratios of $0.85$--$1.03$, whereas the $13$ distinct cells of the full-depth scan have medians of $0.44$--$0.89$ and the five deep cells of the rescan $0.43$--$0.52$; a one-sided exact Mann--Whitney test of cell medians against the $20$ null-seed medians gives $p=3.3\times10^{-8}$ for the full-depth scan and $p=1.9\times10^{-5}$ for the deep rescan (every Mann--Whitney test in this paper is exact and one-sided, on cell medians over all $120$ triples). The regularity also survives the most plausible counterexample generator we could construct: two SimCLR encoders (ResNet-20, trained \emph{without labels}, hence free of the cross-entropy pressure toward neural collapse) develop strong certified pairwise separation in their deep layers---mean quotient $\approx0.30$ at the penultimate stage, $45/45$ pairs certified, remarkable for a label-free objective---yet their triple scans likewise return zero candidates: the largest joint-vs-pairwise excess in either seed is $+0.04$ (seed~0, stage~1) and $+0.03$ (seed~1, stem), both within permutation noise of the pairwise level and neither in a deep layer, and in the deep layers every triple sits below its pairwise maximum (largest excess $-0.07$ and $-0.08$). Where joint quotients are large (early layers), they match the pairwise levels---triples inherit their pairs' entanglement. Only a $k$-fold statistic could have asked the question; the answer is that for these trained classifiers the pairwise entanglement matrix carries almost all of the class-overlap structure the intersection invariant can see, with exceptions that are small, mostly in early and middle layers, and listed.

\emph{Pairwise sufficiency on a target.} Dominance is a statement about the invariant; the operational question is whether the joint term carries anything about the network's \emph{errors} that the pairs do not. A confusion matrix is a pairwise object, so the error mass inside a class triple $\{a,b,c\}$ is the sum of its three pairwise confusions on the test set, and we ask whether the triple quotient predicts it beyond the three pairwise quotients. Over the $18$ distinct layer cells with a triple scan (the three architectures at every layer and the dimension-matched deep rescan, a ResNet's penultimate row counted once with its stage~3, as above), a leave-one-out regression on the sorted pairwise quotients reaches $R^2=0.71$ in the deep stage of the CIFAR-10 ResNet-56 and $0.76$ at the ViT's penultimate layer. Adding the triple quotient does not improve that prediction: the leave-one-out $R^2$ changes by a median of $-0.004$, rises in only $4$ of the $18$ cells, and by at most $+0.04$ (one early ViT block). The triple quotient is not, however, unrelated to the target once the pairs are known. Its partial rank correlation reaches $p<0.05$ in $7$ of the $18$ cells, far more than chance, with a sign that follows depth: the five negative ones are all deep layers and the two positive ones early ($\rho$ from $-0.32$ to $+0.29$). That is a residual association too weak to register as out-of-sample gain, measured against a target that is itself a sum of pairwise confusions. Within this invariant the pairwise matrix is sufficient to \emph{predict} the error structure; what the deep-layer negative residual reflects we leave open.

\emph{Why dominance is the default, and why it is not a theorem.} By witness-region nesting, $\bigcap_{i\in T}\U(X_i;r)\subseteq\bigcap_{i\in S}\U(X_i;r)$ for $S\subseteq T$: the triple overlap lives \emph{inside} each pairwise overlap, so a triple can be active only at scales where all three pairwise overlaps are nonempty, and a triple entangled while its pairs test null would require \emph{sustained} Euler cancellation in all three---the failure mode the guard monitors, which occurred $0$ times in $52{,}650$ measurements (E4). Nesting also explains where the exceptions sit: at ratios of $1.0$--$1.4$, either in early layers whose pairwise quotients are near the null, where the triple's null is smallest relative to its signal, or in the deep stage of the CIFAR-100 subset, where the ratio divides one near-zero quotient by another. It does not make dominance a theorem, and we checked: $\chi$ is not monotone under inclusion, so cutting a pairwise overlap by a third ball union can \emph{create} components, and the permutation null of a triple is far smaller than a pair's, which inflates triple quotients. An adversarial search over anisotropic Gaussian triples in $\R^3$ ($n=30$ per class, $B=99$), constrained to the guard-certified regime, finds $\widehat E_{\mathrm{triple}}/\max_{\mathrm{pairs}}\widehat E=1.98$ at the optimum and a median of $1.24$ over $20$ fresh draws of the same geometry (above $1$ in $85\%$). Dominance is therefore a property of these representations, not of the geometry; and because an unstructured family is scored ``dominated'' $83\%$ of the time (E9), the distance of the ratio distribution from the null floor, not the rate, is the finding. Experiment~E9 asks whether training induces it and finds it already present at initialization, deepened by training on the true labels, and manufactured from an exchangeable start by random-label memorization. The converse lesson stands: genuine higher-order entanglement, if it exists in representations, is structurally invisible to \emph{intersection}-based invariants and must be sought in complement or linking invariants, whose complete $k$-cloud form \citet{kawamura2026intersection} leave open.
\end{experiment}

\begin{figure}[t]
\centering
\includegraphics[width=0.6\textwidth]{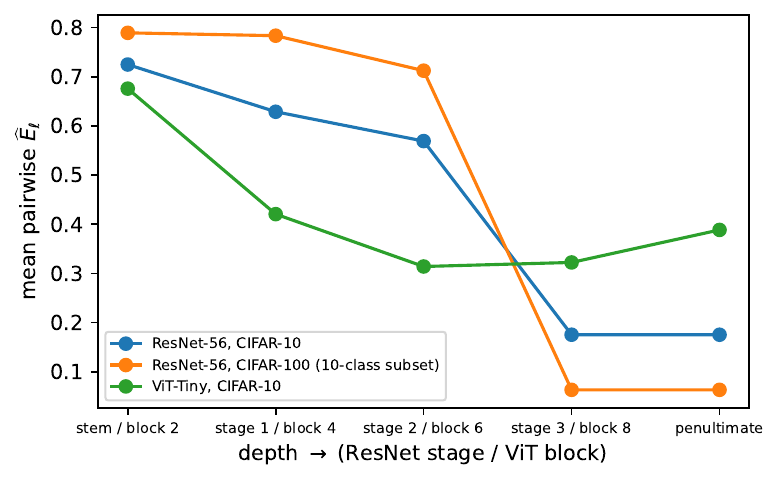}
\caption{Depth-wise disentanglement at the final epoch: mean pairwise quotient per stage for ResNet-56 on CIFAR-10 and on the CIFAR-100 subset, and for the ViT on CIFAR-10. The ResNets drop sharply at the deep stage (their penultimate point repeats stage~3); the ViT is weaker and non-monotone.}
\label{fig:e2_depth}
\end{figure}

\begin{figure}[t]
\centering
\includegraphics[width=0.95\textwidth]{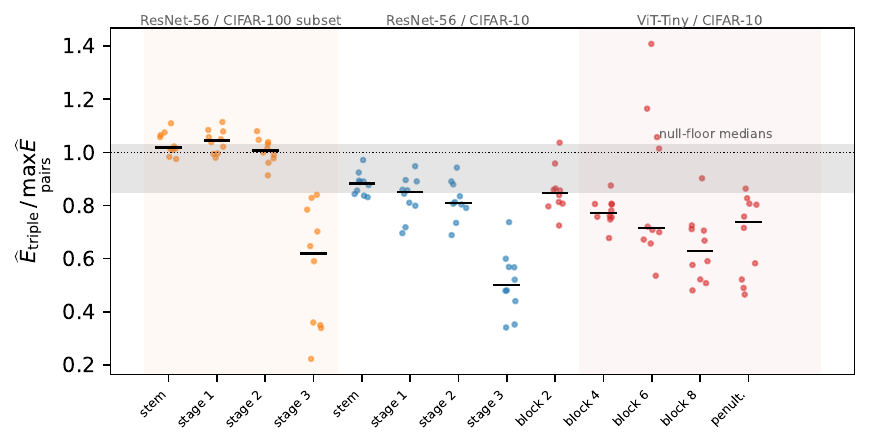}
\caption{The higher-order scan: $\widehat E_{\mathrm{triple}}/\max_{\mathrm{pairs}}\widehat E$ for all $120$ triples in every layer of the three architectures ($d_0=4$), one strip per distinct cell, bar at the median. Dotted line: ratio $1$; grey band: per-seed medians of the null floor (E9). Exceedances sit in all four stages of the CIFAR-100 ResNet, mostly the early ones, and in block~6 of the ViT.}
\label{fig:e2_triple}
\end{figure}

\begin{experiment}{E3 --- training-dynamics trajectories at scale}
Train ResNet-20 on CIFAR-10 for 5 seeds, checkpointing at epochs $0,1,2,4,8,16,32$ and every $10$ epochs from $40$ to $80$; track $\widehat E_\ell$ at every residual stage for every checkpoint, exploiting the no-reduction cost (P2) to a scale impractical for mixup barcodes. Report paired-test significance (Section~\ref{sec:paired}) for each epoch-to-epoch step.

\emph{Result.} Across $12{,}375$ paired epoch-to-epoch tests ($5$ seeds $\times$ $5$ stages $\times$ $45$ pairs $\times$ $11$ checkpoint steps) on the scale-free fold statistic of Section~\ref{sec:paired}, $2{,}707$ steps are BH-significant, $2{,}026$ of them decreases: training disentangles, and the claim is certified stepwise rather than read off a curve (Figure~\ref{fig:nn_training}). The significance sits at the start and in the deep stage: $687$ significant steps at epoch $0\to1$ and $1{,}728$ within the first eight epochs, and over the whole run the deep stage carries $873$ certified decreases against $42$ at the stem. The $681$ significant increases are of two kinds: at epoch $0\to1$ about a fifth of the pairs in every stage re-entangle as the random features are overwritten ($202$), and in the middle stages individual pairs re-entangle transiently between epochs $1$ and $4$ ($130$ and $80$) while the stage means keep falling. A smaller late wave exists and follows the learning-rate schedule, at BH level only: under the cosine schedule, $136$ certified decreases at epochs $50\to60$ and $98$ at $60\to70$, every one of them in the deep stage; under a step schedule with drops at epochs $40$ and $60$ ($2$ seeds, identical protocol) the wave moves to $40\to50$ ($166$ decreases, $160$ in the deep stage); under a constant learning rate no step after epoch $32$ carries more than $49$ decreases, and the two late steps carry $18$ and $32$ against $36$ and $16$ increases. Learning-rate decay, not training time, triggers a late reorganization of the deep stage, but the effect is an order of magnitude smaller than the initial separation and does not survive the Benjamini--Yekutieli correction (Section~\ref{sec:estimation}). Disentanglement is depth-graded exactly as in E2: over training, the mean quotient at the penultimate stage falls $0.78 \to 0.16$ while the stem barely moves ($0.76 \to 0.73$)---early layers never disentangle; deep layers do almost all the work.

\emph{What the raw statistic certified instead.} The same $12{,}375$ tests on the raw profile mass $\int|\Dx|\,dr$---the fold statistic of our first campaign---return $5{,}633$ significant steps, and $1{,}125$ of them, every significant step at epoch $0\to1$, are \emph{increases}, at the very step where every quotient falls. The permutation-null mean of the mass grows twentyfold at that step in the deep stage (the features spread out), and the raw mass follows the scale rather than the interaction. The late `wave' it certified sat in the \emph{stem}---$225$ of $225$ pair--seed tests at $50\to60$---exactly where the quotient is flat at $0.74$ and the feature norm is shrinking under weight decay. The quotient and the certificates of E1--E2, which divide by the same-grid null, were never affected; the paired comparisons were, and the statistic is dimensionless throughout this paper. We report the episode because it is the failure mode that any trajectory analysis on an unnormalized geometric statistic will meet. 

\emph{The same test on cheap statistics, same folds.} We ran the identical paired analysis on the Fisher ratio, the mean-distance ratio, the energy distance and the $5$-nearest-neighbour cross-class rate of E1 on the same fold clouds. The neighbour rate tells the ECP's story: $3{,}456$ significant steps, $2{,}933$ of them disentangling, $2{,}356$ within the first eight epochs, the cosine wave at $50\to60$ ($173$ decreases, $146$ in the deep stage) and the step-schedule wave at $40\to50$ ($202$, $170$ deep), and no localized late wave at constant learning rate. The late wave is therefore a property of the representation, not of the measure. The three mean-based statistics behave differently: they certify changes at almost every step in both directions ($9{,}448$ significant under the cosine schedule, $6{,}017$ disentangling and $3{,}431$ entangling; $468$ against $389$ even at $70\to80$, in every stage including the stem, where neither the quotient nor the neighbour rate moves), because a drift of the class means registers as a change whether or not the overlap changes. For timing \emph{when} classes separate, the overlap statistics---topological or nearest-neighbour---resolve it and the mean-based ones do not.

\emph{Result: cost, on identical clouds.} On the final-epoch representations of the $104$ main-suite networks we ran the mixup barcode---the authors' reference implementation, vendored unmodified---on exactly the clouds the ECP measures ($m=200$, $d_0=5$, both inclusion directions since the barcode is asymmetric, degrees $0$--$1$): $31.7$\,s per pair against $3.7$\,s for the full ECP profile (three Alpha-complex sweeps), a mean $8.5\times$ per layer (range $3.9$--$15.2\times$), $206$ vs.\ $24$ CPU-hours in total; the gap is the image-persistence matrix reduction against the reduction-free sweep, and it compounds with $B$ under permutation testing---which the barcode does not support. At the barcode's cost, the $12{,}375$ paired tests of this experiment alone would be infeasible.
\end{experiment}

\begin{figure}[t]
\centering
\includegraphics[width=0.95\textwidth]{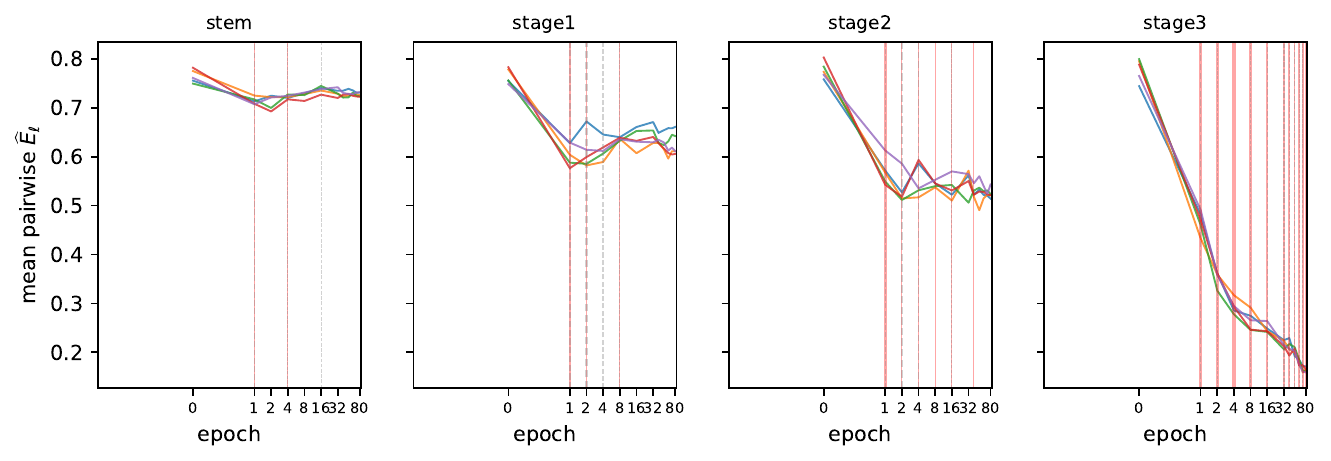}
\caption{Training dynamics, certified: mean pairwise quotient per stage for ResNet-20 on CIFAR-10, one curve per seed (stage~3 is the penultimate representation). Vertical marks are epochs receiving BH-significant paired steps on the scale-free statistic---red decreases, grey increases, weighted by count. Deep stages disentangle; the stem does not.}
\label{fig:nn_training}
\end{figure}

\begin{experiment}{E4 --- calibration, and the guard}
Apply both one-sided tests to every layer/pair and the paired test to every trajectory step; Verify Type-I calibration empirically: on label-shuffled (null) data, the rejection rate at level $\alpha$ must match $\alpha$ across $500$ null replicates, for both tails and for the sign-flip test. Report the cancellation-guard firing rate ($r^\ast$ vs.\ $\Dx$, Section~\ref{sec:tests}) across all measured clouds.

\emph{Result (calibration).} All three tests are calibrated. Over $500$ null replicates, empirical rejection rates at $\alpha\in\{0.01, 0.05, 0.10\}$: $p_\uparrow$ gives $(0.010, 0.044, 0.090)$, $p_\downarrow$ gives $(0.014, 0.064, 0.092)$, and the sign-flip paired test gives $(0.006, 0.052, 0.100)$---each within binomial error of nominal (Figure~\ref{fig:e4_calibration}).

\emph{Result (guard).} Across all $52{,}650$ pairwise measurements of the campaign ($38{,}475$ in the main suite, audited by this experiment's pipeline, plus $14{,}175$ in the follow-up suite, audited post hoc from the archived measurement records) the cancellation guard fired \emph{zero} times ($42{,}570$ distinct clouds, the stage-3 and penultimate rows of each ResNet being the same clouds measured twice). The guard monitors one failure pattern---silence of $\Dx$ on the quarter of the scale range immediately past first contact while $\hat r^\ast$ certifies that the offsets already meet---so a zero firing rate says that this pattern never occurred, not that Euler cancellation is impossible elsewhere in the range, and on pairs whose quotient is far from zero the pattern is unlikely to begin with. What the count establishes is narrower and still useful: no separation certificate in the campaign was issued on a pair whose profile was silent at the scales where the clouds provably interact, which is exactly the false certificate of Theorem~\ref{thm:false}, and the check costs nothing beyond the sweep already run. Because the campaign window begins at the contact scale, where a sampled annulus is still a set of contact points, we re-evaluated every cell whose features remained cached ($45{,}225$ of the $52{,}650$; the $34$ population, SimCLR and finetuned-ViT files whose cached features had been purged are covered by the archived window only) under three criteria: the campaign window, the same window shifted one grid point past contact, and the plateau form of Section~\ref{sec:tests}. None fired on any cell, and the recomputation reproduced every archived guard decision. The profiles say why: on class clouds the overlap is a set of blobs, and $\Dx$ first returns to zero a median of $78$ grid points after contact, on a live range of $181$; $1{,}797$ cells return to zero within the first quarter, and in none of them does the zero persist.
\end{experiment}

\begin{figure}[t]
\centering
\includegraphics[width=0.45\textwidth]{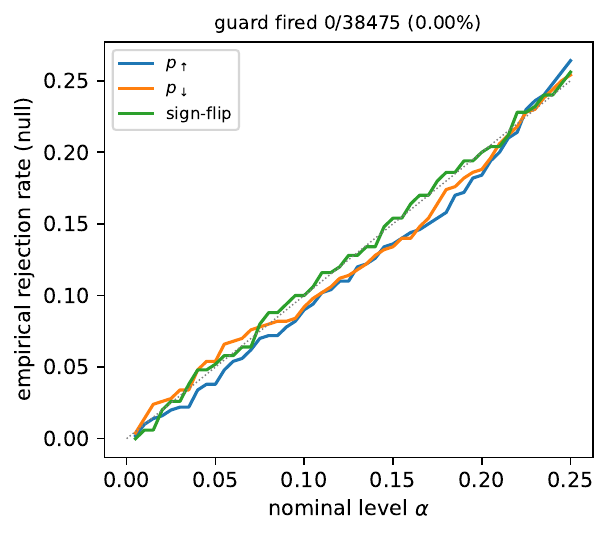}
\caption{Type-I calibration: empirical null rejection rate against the nominal level for both one-sided tests and the sign-flip test ($500$ null replicates). All three track the diagonal.}
\label{fig:e4_calibration}
\end{figure}

\begin{experiment}{E5 --- predicting generalization (exploratory)}
Train a population of $96$ CIFAR-10 models crossing depth $\{20,32,44,56\}$ $\times$ width $\{1\times,2\times\}$ $\times$ weight decay $\{0,5\!\cdot\!10^{-4}\}$ $\times$ augmentation $\{$on, off$\}$ $\times$ 3 seeds; regress held-out test accuracy on disentanglement features (final-epoch $\widehat E_\ell$ per stage, trajectory summaries, triple-entanglement counts) and report predictive $R^2$/rank correlation against total mixup, CKA, linear probes, and neural collapse.

\emph{Result (two-sided).} Over the $96$-model population the null-normalized quotient features are informative but weak (leave-one-out $R^2=0.34$, Spearman $\rho=0.65$ for test accuracy, and similarly for the generalization gap), while the baselines are near-ceiling ($R^2=0.99$, $\rho=0.99$), dominated by linear-probe accuracy, which is nearly the prediction target itself; total mixup---the authors' vendored reference implementation on identical clouds, direction-symmetrized, degrees $0$--$1$---predicts far better than the quotient ($R^2=0.84$, $\rho=0.91$). The explanation is normalization. The quotient divides out the density and the feature scale (the raw mass has units of feature length, Section~\ref{sec:paired}), and that scale is what a predictor of accuracy uses: the \emph{unnormalized} ECP features (per-layer mean and maximum profile mass $\int|\Dx|$ and mean first-interaction scale $r^\ast$, all from the same sweep) reach $R^2=0.94$, $\rho=0.96$ for test accuracy and $R^2=0.92$ for the gap, above total mixup (Table~\ref{tab:e5}). None of it survives the probes: adding either ECP family, or total mixup, to the baselines leaves them unchanged ($R^2=0.985$--$0.988$ against $0.987$). For raw prediction on this population, no topological interaction summary---certified or descriptive, ours or theirs---carries information beyond what trivially cheap probes already provide; the measurement's value lies in the certified structural claims of E1--E4 and E9--E10, not in generalization prediction.

\emph{Result: the quotient's departure from the cheap statistics.} Scoring each (model, layer) cell by how far the quotient departs from what the neighbour rate or the Fisher ratio of E1 predicts, and regressing test accuracy, generalization gap, expected calibration error, negative log-likelihood, and accuracy under four on-the-fly corruptions on those departures, alone and on top of the baselines, changes no leave-one-out $R^2$ by more than $0.03$. The departure predicts nothing the cheap statistics do not; the deep-stage neighbour rate is itself the strongest calibration predictor in the population (Spearman $0.94$ with calibration error, above probe accuracy), and corruption robustness, which clean accuracy barely predicts (Spearman $0.30$), is predicted weakly by mid-stage entanglement ($R^2=0.34$) whichever statistic measures it. The outcome sits comfortably in a literature that has repeatedly found complexity measures to correlate weakly or unstably with generalization once models are swept systematically \citep{jiang2020fantastic}; the topological entrant with the strongest claim measures the optimizer's trajectory rather than the representation \citep{birdal2021intrinsic}.
\end{experiment}

\begin{table}[t]
\centering
\begin{tabular}{llrrr}
\toprule
target & features & $n_{\mathrm{feat}}$ & LOO $R^2$ & Spearman \\
\midrule
test\_acc & ecp & 16 & 0.337 & 0.655 \\
test\_acc & baselines & 15 & 0.987 & 0.993 \\
test\_acc & ecp+baselines & 31 & 0.985 & 0.993 \\
test\_acc & ecp-raw & 15 & 0.935 & 0.963 \\
test\_acc & ecp+ecp-raw & 31 & 0.925 & 0.956 \\
test\_acc & baselines+ecp-raw & 30 & 0.988 & 0.994 \\
test\_acc & total-mixup & 10 & 0.841 & 0.908 \\
test\_acc & baselines+total-mixup & 25 & 0.987 & 0.993 \\
gen\_gap & ecp & 16 & 0.340 & 0.662 \\
gen\_gap & baselines & 15 & 0.981 & 0.990 \\
gen\_gap & ecp+baselines & 31 & 0.979 & 0.989 \\
gen\_gap & ecp-raw & 15 & 0.918 & 0.955 \\
gen\_gap & ecp+ecp-raw & 31 & 0.906 & 0.946 \\
gen\_gap & baselines+ecp-raw & 30 & 0.979 & 0.987 \\
gen\_gap & total-mixup & 10 & 0.818 & 0.899 \\
gen\_gap & baselines+total-mixup & 25 & 0.981 & 0.990 \\
\bottomrule
\end{tabular}

\caption{E5 exploratory generalization prediction over $96$ models: leave-one-out $R^2$ and Spearman correlation of ridge regression by feature family.}
\label{tab:e5}
\end{table}

\begin{experiment}{E6 --- steering disentanglement with a differentiable surrogate (summary; Appendix~\ref{app:e6})}
Can the measure be turned around to \emph{steer} training? Since $\Dx$ is integer-valued, we optimize a smooth surrogate: on a mini-batch, the mean over cross-class pairs $\mathcal{C}$ of a multi-scale Gaussian kernel of their feature distance,
\[
    \mathcal{L}_{\mathrm{dis}} \;=\; \frac{1}{|\mathcal{C}|}\!\!\sum_{(x,y)\in\mathcal{C}} \frac{1}{|\mathcal{S}|}\sum_{s\in\mathcal{S}} \exp\!\Big(\!-\frac{\|f(x)-f(y)\|^2}{2\,s\,\sigma^2}\Big),
\]
where $\sigma^2$ is the median \emph{within}-class squared distance in the batch, so that the \emph{value} of $\mathcal{L}_{\mathrm{dis}}$ is scale-invariant like $\Dx$ and cannot be reduced by inflating feature norms; invariance of the \emph{gradient} additionally requires $\sigma^2$ to stay in the computational graph. We train ResNet-20 on CIFAR-10 at the E3 recipe with $\mathcal{L}=\mathcal{L}_{\mathrm{CE}}+\lambda\,\mathcal{L}_{\mathrm{dis}}$ for $\lambda\in\{0,0.5,2,8\}$ and judge the surrogate by whether the \emph{certified} penultimate quotient, measured afterwards by the non-differentiable ECP, moves in $\lambda$. Both implementations of the surrogate---with the bandwidth $\sigma^2$ detached from the graph and with it inside---\emph{lower} the surrogate while \emph{raising} the certified entanglement, monotonically in $\lambda$ ($0.159\to0.344$ detached; $0.167\to0.488$ attached, at $\lambda\le2$), each for an identifiable reason: the detached surrogate is gamed by feature-norm inflation, the attached one by within-class dispersion. E6 is therefore a measurement result rather than a regularizer: a differentiable surrogate for the interaction quotient is not yet available, and the certificate is what exposes the failure. Appendix~\ref{app:e6} reports the full campaign, with accuracies and robustness.
\end{experiment}

\begin{experiment}{E7 --- the interaction spectrum (summary; Appendix~\ref{app:e7})}
On the final-epoch E2 models and ResNet-20 we compute the full spectrum of floors $\widehat E(C_j)$, $j=2,\dots,5$ (``at least $j$ classes meet''), the top floors $j\in\{8,9,10\}$ of all ten classes, and the onset vector $r^\ast_2\le\dots\le r^\ast_{10}$ per layer. The spectrum decays monotonically in $j$ in every layer of every model, at a depth-graded rate (mean step ratio $\approx0.88$ at the stem, $\approx0.61$ at the deepest stage); the top floors are negligible in the deepest ResNet stages (quotients $0.015$--$0.07$ against $0.44$--$0.73$ at the stem; $B=19$, so magnitudes rather than sharp certifications); and the onsets stretch with order, the ten-fold overlap beginning at twenty times the pairwise scale in the deep stage. This is the graded, scale-domain form of E2's dominance; Appendix~\ref{app:e7} gives the full spectra and onset curves.
\end{experiment}

\begin{experiment}{E8 --- robustness to the projection dimension (summary; Appendix~\ref{app:e8})}
Re-measuring the depth conclusions of E2 at $d_0\in\{3,4,6\}$ on all three E2 models and the pretrained ViT-B/16 leaves every qualitative conclusion unchanged: the depth grading, the ViT's non-monotone profile and the pretrained ViT's monotone one persist at every $d_0$, deep layers stay fully certified separated, and the pair ranking against the $d_0=5$ reference has median Spearman $0.91$, $0.96$, $0.97$ at $d_0=3,4,6$, the few lower values ($0.56$--$0.69$) occurring only at $d_0=3$ among pairs already indistinguishable from zero. Appendix~\ref{app:e8} gives the numbers.
\end{experiment}

\begin{experiment}{E9 --- is pairwise dominance present at initialization?}
Is dominance a property of PCA-projected class clouds of natural images that training merely preserves, or something training produces? We rebuild the initializations of the E3 ResNet-20 and the E2 ResNet-56 and ViT exactly as the training script does (seeded construction, before the first optimizer step), push the CIFAR-10 test set through them, and run the full pairwise matrix and the $\binom{10}{3}$ triple scan at the same $d_0$ on every stage boundary and on the raw normalized input ($m=200$, $B=49$ for both orders, so the ratio is not biased by unequal precision), scored exactly as in E2 and compared against the dimension-matched trained cells. As a control in the opposite direction we train the same ResNet-20 to \emph{memorize} a fixed random relabeling of CIFAR-10 \citep{zhang2017understanding} (no augmentation, no weight decay, $150$ epochs; train accuracy $0.9999$, chance $0.0995$ on the true test labels) and measure it on the training images with the labels it fitted, where the ten ``classes'' are exchangeable by construction until the network makes them otherwise.

\emph{Result: present at initialization, deepened by training.} At the initialized penultimate layer ($d_0=5$) every triple of every architecture is dominated: $\widehat E_{\mathrm{triple}} \le \max_{\mathrm{pairs}}\widehat E$ in $120/120$ for ResNet-20, ResNet-56 and the ViT alike, with median ratios $0.76$, $0.76$ and $0.78$, $90$th percentiles $0.87$--$0.89$, and maxima below $1$ ($0.95$--$0.98$). The same holds at every stage boundary of the ResNet-20 at $d_0=4$ ($120/120$ at the stem and both middle stages, medians $0.74$--$0.76$). The pairwise quotients are themselves already off the null before training (means $0.74$--$0.80$; $34$--$40$ of $45$ pairs certified separated), and the raw normalized pixels give the same pairwise level as the random ResNet-20 ($0.742$ against $0.744$) and the same dominance ($119/120$, median ratio $0.73$, one exceedance at $1.04$): CIFAR pixels carry the pairwise structure already, and a random network passes it through essentially unchanged. Training does not create the pairwise structure; it deepens it, the median ratio falling from $0.76$--$0.78$ at initialization to $0.47$ at the trained deep layers of E2.

\emph{The bias floor.} The dominance score compares one quotient against the maximum of three. Under exchangeable labels every quotient sits near $1$ with permutation noise, so the maximum is biased upward and the ratio below $1$, and an unstructured triple family will register as ``dominated'' much of the time. We measured that floor on two nulls at the same $B$, ten seeds of four clouds each: four clouds drawn from one isotropic Gaussian in $\R^5$ give median ratio $0.93$ ($33/40$ dominated, maximum $1.09$; per-seed medians $0.88$--$1.03$), and the initialized features with labels permuted give $0.97$ ($33/40$, maximum $1.05$; per-seed medians $0.85$--$1.01$). Every initialization cell and every trained deep cell sits well below this floor (exact one-sided Mann--Whitney test \citep{mann1947test} of cell medians against the $20$ null-seed medians: $p=3.3\times10^{-8}$ for the ten initialization cells, $p=1.9\times10^{-5}$ for the five trained deep cells of E2; the cells are layers of a few networks and the null family has four triples per seed, so we read these as descriptive comparisons of location rather than as independent replicates), so the dominance reported here and in E2 is signal rather than the artifact, and the null rate of $83\%$ is what a genuinely unstructured family produces, against $100\%$ at initialization and $99.5\%$ trained.

\emph{Result: memorization manufactures it, in the last stage only.} The random classes are exchangeable at initialization and remain so through the stem and the first two stages of the trained network: mean quotients $1.02$, $1.01$, $0.99$, with $0$, $1$ and $1$ of $45$ pairs certified separated---the false-positive rate at $\alpha=0.05$---and $p_\uparrow$ firing on $5$, $3$ and $0$. The last stage separates all of them: mean quotient $0.26$, $45/45$ pairs certified, and every one of the $120$ triples dominated at median ratio $0.34$ ($90$th percentile $0.46$, maximum $0.59$; pairs at $B=199$, triples screened at $B=49$ with the ten largest escalated to $B=999$; the pairs' larger $B$ only sharpens the denominator, which biases the ratio upward, so the comparison is conservative), stricter than the trained natural classes of E2 ($0.47$) and far below the null floor. Two things follow. Fitting arbitrary labels creates pairwise structure where none existed, so training can \emph{manufacture} the dominance as well as deepen it; and it does so in the deepest stage alone, every earlier layer staying exchangeable with respect to the memorized labels. The second is a certified form of the observation that memorization is concentrated in the later layers \citep{stephenson2021geometry}, made here with a test per layer rather than a capacity curve.

\emph{Reading.} Pairwise dominance is not a fact about learning natural classes and not forced by geometry: it is present in the pixels and at initialization, sharpened by training on the true labels, manufactured outright when the network memorizes arbitrary ones, and reproduced on frozen language models (E10). Within this invariant, whatever a network does to a labeled cloud family, it does mostly pairwise.
\end{experiment}

\begin{experiment}{E10 --- outside vision: two frozen language models}
Every representation above comes from vision: networks we trained or finetuned, their initializations, and raw pixels. Our concentration diagnostic found that language-model features, unlike deep vision features, keep a usable scale window at every depth (Appendix~\ref{app:concentration}), which makes them the one wide representation
on which the full certified protocol should run unmodified. We take
BERT-base \citep{devlin2019bert} and GPT-2 \citep{radford2019language}, frozen,
mean-pooled over real tokens, on ten classes of 20~Newsgroups (test split,
headers, footers and quotes stripped, posts under $20$ words dropped,
duplicate posts removed; $3{,}358$ and $3{,}363$ documents, at least $279$ per
class), and run the E9 protocol: the $45$-pair matrix and all $120$ triples
at $d_0=5$, $m=200$, $B=49$, on the final layer of each and on BERT's middle
block; and repeat it on a second corpus, ten DBpedia ontology classes ($4{,}000$ Wikipedia abstracts, $400$ per class).

\emph{Result.} The pipeline runs end to end on both, and the same pairwise dominance holds in
both (Table~\ref{tab:lm}).
\begin{table}[h]
\centering\small
\caption{E10: two frozen language models on two corpora ($d_0=5$, $m=200$, $B=49$). ``Pairs certified'' is the number of the $45$ class pairs certified separated; ``dominated'' the number of the $120$ triples whose joint quotient is at or below their largest pairwise quotient; ``median ratio'' the median of $\widehat E_{\mathrm{triple}}/\max_{\mathrm{pairs}}\widehat E$ (null floor: per-seed medians $0.85$--$1.03$).}
\label{tab:lm}
\begin{tabular}{@{}lllcccc@{}}
\toprule
model & corpus & layer & pairs certified & mean pair $\widehat E$ & dominated & median ratio\\
\midrule
BERT-base & 20 Newsgroups & block 6 of 12 & $43/45$ & $0.55$ & $120/120$ & $0.57$\\
BERT-base & 20 Newsgroups & final & $43/45$ & $0.46$ & $120/120$ & $0.46$\\
GPT-2 & 20 Newsgroups & final & $33/45$ & $0.73$ & $118/120$ & $0.68$\\
BERT-base & DBpedia & block 6 of 12 & $45/45$ & $0.15$ & $120/120$ & $0.36$\\
BERT-base & DBpedia & final & $45/45$ & $0.16$ & $120/120$ & $0.42$\\
GPT-2 & DBpedia & final & $44/45$ & $0.46$ & $120/120$ & $0.39$\\
\bottomrule
\end{tabular}
\end{table}
BERT's final layer lands on the trained deep-vision median ratio of E2 ($0.46$ against $0.47$) on an encoder that never saw these labels, and its middle block is less separated and less strictly pairwise than its final layer---the depth gradient of E2 in a language model. GPT-2 separates the newsgroup topics far less (its mean-pooled final states have an effective rank below $6$, Table~\ref{tab:concentration}), yet its triples still sit well below the null floor of E9 (median $0.68$ against null-seed medians of $0.85$--$1.03$), its two exceedances at $1.01$ and $1.02$. Across the six language-model cells every median sits below every null-seed median (exact Mann--Whitney $p=4.3\times10^{-6}$). On DBpedia, BERT certifies every pair at both depths with mean quotients of $0.15$--$0.16$ and dominates every triple at median ratios of $0.36$ and $0.42$, and even GPT-2 certifies $44$ of $45$ pairs and dominates every triple (median $0.39$). Two models on two corpora is what we claim, with every cell landing on the same pairwise structure and the amount of separation tracking how much class signal the representation carries.
\end{experiment}

\begin{experiment}{E11 --- what changes certified disentanglement: depth, width, weight decay, augmentation}
The E5 population is a factorial design: ResNet depth $\{20,32,44,56\}$ $\times$ width $\{1,2\}$ $\times$ weight decay $\{0, 5\cdot10^{-4}\}$ $\times$ augmentation $\{\text{off},\text{on}\}$ $\times$ three seeds, all measured at every stage of the final epoch. Two models that differ in exactly one factor share the same test images in the same folds, so the paired test of Section~\ref{sec:paired} applies to a matched pair of \emph{models} exactly as it does to two checkpoints. For each factor contrast we run two tests. On the scale-free fold mass: within every matched pair and stage, the sign-flip test per class pair, BH within the $45$-pair family (\emph{mass} in Table~\ref{tab:factorial}; $24$ or $48$ matched pairs, $1{,}080$ or $2{,}160$ tests per stage). On the interaction quotient: for each class pair and stage, the $24$ or $48$ matched-model differences of $\widehat E$ are exchangeable in sign under the factor's null within each matched pair, so a sign-flip test over them (Monte Carlo, $19{,}999$ flips, with the $+1$ convention of Proposition~\ref{prop:exact}) with BH within the $45$-pair family certifies whether the factor moves the quotient of that pair (\emph{quotient} in the table). The first test asks whether the overlap got smaller; the second whether it got smaller \emph{relative to chance}.

\begin{table}[t]
\centering\footnotesize
\setlength{\tabcolsep}{3pt}
\begin{tabular}{@{}lrrrrrrrrrr@{}}
\toprule
contrast & pairs & $\Delta$acc & \multicolumn{2}{c}{stem} & \multicolumn{2}{c}{stage 1} & \multicolumn{2}{c}{stage 2} & \multicolumn{2}{c}{stage 3} \\
 & & & mass $\downarrow/\uparrow$ & $\widehat E$ $\downarrow/\uparrow$ & mass $\downarrow/\uparrow$ & $\widehat E$ $\downarrow/\uparrow$ & mass $\downarrow/\uparrow$ & $\widehat E$ $\downarrow/\uparrow$ & mass $\downarrow/\uparrow$ & $\widehat E$ $\downarrow/\uparrow$ \\
\midrule
depth 20$\to$32 & 24 & +0.005 & 0/0 & 0/0 & 38/39 & 0/0 & 0/18 & 0/2 & 25/16 & 0/0 \\
depth 32$\to$44 & 24 & -0.003 & 0/18 & 0/0 & 72/57 & 0/0 & 11/146 & 0/12 & 1/30 & 0/0 \\
depth 44$\to$56 & 24 & -0.002 & 10/21 & 0/0 & 52/71 & 0/0 & 106/54 & 0/0 & 26/25 & 0/0 \\
depth 20$\to$56 & 24 & -0.000 & 15/29 & 0/3 & 88/53 & 0/0 & 11/134 & 0/25 & 39/36 & 0/11 \\
width 1$\to$2 & 48 & +0.014 & 2/126 & 0/0 & 302/19 & 0/0 & 114/89 & 0/6 & 297/26 & 1/0 \\
weight decay 0.0$\to$0.0005 & 48 & +0.033 & 133/2 & 3/11 & 40/344 & 0/0 & 60/119 & 9/6 & 1743/12 & 2/32 \\
augmentation False$\to$True & 48 & +0.055 & 19/5 & 0/2 & 127/263 & 11/2 & 110/162 & 24/1 & 1307/46 & 44/0 \\
\bottomrule
\end{tabular}

\caption{E11. For each factor contrast (low $\to$ high) and stage: certified decreases$/$increases of the scale-free overlap mass among the $1{,}080$ or $2{,}160$ paired tests (\emph{mass}), and the number of the $45$ class pairs whose interaction quotient is certified to decrease$/$increase across the matched models ($\widehat E$); the mean quotient changes are in the text. $\Delta$acc is the mean change in test accuracy across the matched pairs.}
\label{tab:factorial}
\end{table}

\emph{Result: augmentation is the factor that disentangles.} Turning augmentation on lowers the deep-stage quotient of $44$ of the $45$ class pairs (mean $0.230\to0.171$), of $24$ pairs at stage~2 and $11$ at stage~1, and of none at the stem; the overlap mass falls in $1{,}307$ of $2{,}160$ paired tests at stage~3. It is also the largest accuracy gain in the design ($+0.055$).

\emph{Result: weight decay compresses without separating.} Weight decay produces the largest mass effect in the table---the stage-3 overlap mass falls in $1{,}743$ of $2{,}160$ tests---and yet the quotient \emph{rises} for $32$ of $45$ pairs ($+0.02$): the relabeled null shrinks as much as the overlap does. Weight decay simplifies the geometry of the whole representation, overlap included, and leaves the classes no more separated relative to chance; it raises accuracy by $0.033$ all the same. Relative topological separation and accuracy are different quantities, and this is the contrast that shows it.

\emph{Result: depth and width do not disentangle.} From depth $20$ to $56$ the quotient rises for $25$ pairs at stage~2 and $11$ at stage~3, by $0.02$--$0.04$, with no change in accuracy; the adjacent contrasts are null except $32\to44$ at stage~2 ($12$ pairs). The depth grading of E2 is a property of a stage's position in the network, not of how many stages there are. Doubling the width moves the quotient of at most $6$ pairs at any stage, while shrinking the mass at stages~1 and~3 ($302$ and $297$ of $2{,}160$). At the stem no factor lowers the quotient of more than $3$ pairs, and only weight decay raises it, for $11$.

\emph{Reading.} The measure separates two things regularization does. Augmentation pulls the classes apart relative to what a random relabeling of the same clouds would show; weight decay and width simplify the clouds and their overlap together, which the mass sees and the quotient, correctly, does not. And in a residual network the amount of separation a stage achieves is set by where the stage sits, not by the depth of the network around it.
\end{experiment}

\paragraph{Code and records.} The measurement pipeline, the cluster job scripts, and every measurement record behind the numbers above (893 files) are released \coderelease; checkpoints and cached features regenerate from the seeded configurations. Every figure and table rebuilds from the records alone.

\section{Conclusion and limitations}\label{sec:conclusion}\label{sec:limitations}

The Intersection ECP turns class disentanglement into a symmetric, multi-class, cheap, and statistically certified measurement: guarded one-sided tests certify separation, the paired test supports the comparative claims that trajectory analyses actually make, and the interaction quotient makes the numbers comparable across layers, epochs, and models. The methodological result travels furthest: a paired test on a statistic with units certifies the units. On the raw profile mass ours returned $5{,}633$ significant steps, more than twice the scale-free count, with every significant step of the first epoch pointing the wrong way and the late ``wave'' it certified sitting in the stem, where the quotient is flat and the feature norm is shrinking under weight decay. The remedy is one division, and the diagnosis---compare the certified direction against the permutation-null mean of the same statistic---costs nothing and applies to any geometric trajectory claim. The campaign's headline is structural: \emph{in trained image classifiers, class disentanglement is a certified, depth-graded, early-training phenomenon, and its interaction structure is pairwise almost everywhere we looked}. It is confusability-ranked at the input side ($\rho=0.83$ with confusion, on par with a nearest-neighbour rate), concentrated in deep layers ($0.78\to0.16$ at the penultimate stage against $0.76\to0.73$ at the stem), and significant mostly within the first eight epochs, with a smaller, BH-level late wave in the deep stage that follows the learning-rate schedule; it is weaker and non-monotone in a from-scratch ViT, a data-starvation artifact that pretraining removes. 

Joint entanglement of a triple sits below its strongest pair in $97\%$ of triple--layer cells and $99.5\%$ of deep cells, at median ratios far below the null floor, with the interaction spectrum decaying in order (E7) and the exceptions listed (E2); the structure is already present at initialization, deepened by training on the true labels, manufactured outright in the last stage alone when a network memorizes random ones (E9), and reproduced on two frozen language models (E10). Of the design choices in a $96$-model population, augmentation is the one that separates classes relative to chance, weight decay compresses the overlap without separating, and depth and width do nothing (E11); the unnormalized mass predicts generalization, though not beyond linear probes (E5). Every certified clause of that summary carries a $p$-value or a test against a measured null floor, and the graded summaries (the spectrum at $B=19$, the language-model dominance rows) are reported as magnitudes, which is the point: the topology of representations can be measured with the same statistical hygiene as any other quantity in the empirical sciences---provided the statistic is scale-free, as E3 shows the hard way. The open theory is an analytic null for the profile, which would calibrate the tests without permutations, and a hypothesis on labeled clouds under which pairwise dominance becomes a theorem; the guard-certified regime alone does not suffice (E2).

Six limitations bound these claims.

\emph{Dimensionality.} The Alpha complex costs $O(n^{\lceil d/2\rceil})$, so every measurement is made on a PCA shadow of dimension $d_0\le6$, as in \citet{wagner2024mixup}. Proposition~\ref{prop:projection} carries every first-contact scale and guard decision to the ambient representation (the permutation certificates are statements about the shadow), E8 finds every depth conclusion invariant to $d_0\in\{3,4,6\}$, and entanglement enters our claims only as a magnitude comparison, since an overlap in the shadow may be an artifact of it. Two caveats remain. Whether the class supports are well approximated in the projection, which the estimation guarantee assumes, we do not verify; and the triple scan of E2 could in principle miss an ambient triple whose constituent pairs the projection renders apparently entangled, although its mechanism, witness-region nesting, is dimension-free.

\emph{Reach.} The certificate is valid at any width, but its reach degrades in wide representations of high effective rank. On frozen ImageNet features the \emph{scale window}---the fractional gap between the first-percentile cross-class distance and the closest cross pair---narrows along the depth axis of every vision backbone tested ($3.4$--$3.6\times$ from ResNet-50 stage~1 to stage~4), a projected certificate retains only $8$--$26\%$ of the ambient interaction scale, and projection, ambient $\Dx$, and an ambient branch-and-bound for the onset vector all fail on the narrowest layers for one reason, concentration of measure (Appendix~\ref{app:concentration}). Across $55$ vision layer--model cells the window tracks effective rank (Spearman $-0.71$) rather than depth or width as such; language-model features show no such gradient, which is why the full protocol runs unmodified in E10. The claims of this paper therefore concern small trained vision classifiers, their initializations, and two frozen language models. The window is cheap and should be computed before the measurement is interpreted on anything wider.

\emph{Hypotheses of the estimation guarantee.} Proposition~\ref{prop:rate} assumes class supports of positive reach and an $\epsilon$-dense sample of each; class clouds of a trained network are not known to satisfy the first, and we do not verify the second in the projection. The guarantee licenses reading the sample profile as a population property only under those assumptions; the certificates and the paired comparisons, which are exact for the sample, do not depend on them.

\emph{What the cheap alternatives already give.} Where the question is a pairwise ranking or the timing of separation, a nearest-neighbour cross-class rate on the same clouds does as well as the quotient (E1, E3), and the quotient's departures from the cheap statistics predict nothing about accuracy, calibration, or robustness on the E5 population (E5). The case for the topological statistic rests on what those alternatives do not offer---the certificate with a null, the $k$-fold terms, scale resolution---and on the one depth profile where they disagree, the from-scratch ViT's head (E2).

\emph{Sampling and scale.} Results depend on the points per class $m$ and the scale range; stability (P4) bounds the effect in theory, and the E1 sweep over $m\in\{50,100,200,400\}$ and $d_0$ leaves the confusability ranking stable (Spearman median $0.87$, minimum $0.73$ against the reference). Radii are equal across clouds, the one slice a single sweep yields, so a compact class and a diffuse one are thickened alike; a density-offset variant, in which each cloud's balls carry a fixed head start equal to its excess median nearest-neighbour spacing (a weighted Alpha complex on the pooled points, with the offsets recomputed under every relabeling so the test stays exact), moves no E1 quotient by more than $0.04$, ranks the pairs as the diagonal does (Spearman $0.97$), and certifies the same $45/45$. We did not vary the resolution of the $200$-point grid; the integrated functional is insensitive to it in a way the peak statistic is not.

\emph{What $\chi$ discards, and what $\Dx$ cannot steer.} Structurally cancelling overlaps are invisible to $\Dx$ (Section~\ref{sec:blind}); neither the campaign guard nor its plateau form fired on any of the $52{,}650$ measured pairs (E4), so the escalation path was never needed here, but the blind spot is real and the guard is not optional. $\Dx$ is integer-valued, and both differentiable surrogates we tried were optimized away from what the certificate measures (E6), so training-time control remains open. A paired test of the profile across a grokking transition, finally, found the right sign at a magnitude we do not consider a finding (Appendix~\ref{app:grok}).

\appendix
\section{The inference used, with sources}\label{app:inference}

This appendix states the guarantees invoked in Section~\ref{sec:certified} with their sources. The tests are textbook and are cited rather than re-proved; the recovery statement is restated from the foundations paper; what is proved here is what is specific to the Intersection ECP, the blind spot of Theorem~\ref{thm:false} and the consistency of the guard. Throughout, $X=\{x_1,\dots,x_m\}$ and $Y=\{y_1,\dots,y_n\}$ are finite clouds in $\R^d$, $T=\Phi(\Dx(\cdot;X,Y))$ for a functional $\Phi$ of the profile, and $M_r=\U(X;r)\cap\U(Y;r)$. We write $r^\ast(A,B)=\inf\{r:\U(A;r)\cap\U(B;r)\neq\emptyset\}$ for any two sets; for finite clouds and closed balls the infimum is attained and $M_r\neq\emptyset$ iff $r\ge r^\ast$.

\subsection{Exact tests in both directions}\label{app:exact}

\begin{proposition}[Exactness of the two one-sided tests]\label{prop:exact}
Let $H_0$ be that the labels of the pooled sample $Z=X\cup Y$ are exchangeable (in particular, that $X$ and $Y$ are i.i.d.\ from a common distribution). Let $\sigma_1,\dots,\sigma_B$ be drawn uniformly and independently from the splits of $Z$ into parts of sizes $m,n$, independently of the data. Then for every statistic $T$, every $\alpha\in(0,1)$, every $m,n$ and every $B$, the $p$-values of Section~\ref{sec:tests} satisfy $\Pr_{H_0}[p_\uparrow\le\alpha]\le\alpha$ and $\Pr_{H_0}[p_\downarrow\le\alpha]\le\alpha$. The same holds for $k$ clouds with splits into parts of sizes $m_1,\dots,m_k$.
\end{proposition}

This is the Monte Carlo permutation test: conditionally on $Z$, the values $T_{\mathrm{obs}},T^{\sigma_1},\dots,T^{\sigma_B}$ are exchangeable, and the rank $p$-value with the $+1$ in numerator and denominator is valid at every $B$ \citep{dwass1957modified,phipson2010permutation}; the full-group version is the randomization test of \citet[Thm.~15.2.1]{lehmann2005testing}. Ties are handled by the weak inequalities in the definitions. The PCA projection and the scale grid of Section~\ref{sec:setup} are functions of $Z$ alone, so the statement applies verbatim to the projected clouds on their grid.

\subsection{The paired test}\label{app:paired}

Let two conditions (two layers, two epochs, two models) be evaluated on $F$ disjoint folds of the same inputs, and let $D_s=T_1^{(s)}-T_2^{(s)}$ be the difference of their statistics on fold $s$. The comparative null $H_0^{\mathrm{sym}}$ is that the law of $(D_1,\dots,D_F)$ is invariant under coordinatewise sign changes; it holds when the two conditions are exchangeable within each fold given the shared preprocessing, which is the natural null of a paired design. Independence across folds is not required.

\begin{proposition}[Exactness of the sign-flip test]\label{prop:signflip}
Under $H_0^{\mathrm{sym}}$ the two-sided $p$-value $p=2^{-F}\,\#\{\varepsilon\in\{\pm1\}^F:\ |F^{-1}\sum_s\varepsilon_sD_s|\ge|\bar D|\}$ satisfies $\Pr[p\le\alpha]\le\alpha$ for every $F$. Since $\varepsilon$ and $-\varepsilon$ give the same value, $p\ge2/2^F$, so the test can reject at level $\alpha$ only if $2^{F-1}\ge1/\alpha$.
\end{proposition}

This is the randomization test for the group $\{\pm1\}^F$ \citep[Thm.~15.2.1]{lehmann2005testing}, Fisher's sign-flip test \citep[Ch.~III]{fisher1935design}. We take $F=8$, so the floor is $2/256$; the sign of $\bar D$ is read off after the two-sided test (E3, E11).

\subsection{Estimation}\label{app:rate}

\begin{proposition}[Recovery of the population profile; restated from \citet{kawamura2026intersection}]\label{prop:rate}
Let $A,B\subset\R^d$ be compact with reach at least $\tau>0$ and Hausdorff dimensions $k_A,k_B$ with finite, positive Hausdorff measure, sampled i.i.d.\ from distributions with densities bounded below with respect to the corresponding Hausdorff measures, and suppose the population overlap filtration $\rho\mapsto\U(A;\rho)\cap\U(B;\rho)$ has finitely many homological critical values and finite-dimensional homology. Then for $\epsilon,\delta>0$, samples of size $O\big(\epsilon^{-k}(\log(1/\epsilon)+\log(1/\delta))\big)$ from each, $k=\max(k_A,k_B)$, are $\epsilon$-dense in $A$ and $B$ with probability at least $1-2\delta$, and given $\epsilon$-density, $\Dx(r;X,Y)=\chi(\U(A;r)\cap\U(B;r))$ for every $r$ in a band $[r_{\min},r_{\max}]\subset(0,\tau)$ with $\epsilon<\min(r_{\min},\tau-r_{\max})/2$, outside an exceptional set $E_T\cup E_M$ of scales: $E_T$, an $O(\epsilon)$-neighbourhood of the population critical values, and $E_M$, the sample-artifact scales \citep[Thm.~5.6, Cor.~5.8, Rem.~5.9]{kawamura2026intersection}. The rate is in the intrinsic dimension \citep[Rem.~5.7]{kawamura2026intersection}, by the covering argument of \citet[Prop.~3.2]{niyogi2008finding}.
\end{proposition}

Uniform control of $|E_M|$ in the sample size is left open in the foundations paper (their Remark~5.9); Theorem~\ref{thm:false} below carries it as a hypothesis.

\subsection{The blind spot and the guard}\label{app:guard}

\begin{definition}
Under the hypotheses of Proposition~\ref{prop:rate}, $(A,B)$ is \emph{$\chi$-blind on $W\subset(0,\tau)$} if $\U(A;r)\cap\U(B;r)$ is nonempty and has Euler characteristic $0$ for every $r\in W$.
\end{definition}

Because Betti numbers are locally constant away from the critical scales, blindness holds on whole intervals of regular scales, not at isolated ones; annular and torus-shell overlaps are blind over an interval. On a blind pair the population profile vanishes both below first contact (Lemma~3.1 of the foundations paper) and on $W$: the profile of a blind pair is indistinguishable from that of a separated one.

\begin{theorem}[The unguarded separation certificate fails on the blind set]\label{thm:false}
Let $(A,B)$ satisfy the hypotheses of Proposition~\ref{prop:rate} and be $\chi$-blind at every scale of a grid range $G=[0,r_{\max}]$ at which the offsets meet, so that the population profile vanishes on all of $G$. Let $\Phi_G(\Dx)=\int_G|\Dx(r)|\,dr$, and let $m,n\to\infty$ with $m/(m+n)\to\pi\in(0,1)$ and $B\ge1/\alpha-1$ fixed. Assume
\begin{enumerate}
    \item[(i)] $|E_M\cap G|\to0$ in probability and $\sup_G|\Dx(\cdot;X,Y)|$ is stochastically bounded;
    \item[(ii)] the null statistic is bounded away from zero: for a uniformly random split $(X^\sigma,Y^\sigma)$ of the pooled sample, $\Pr[\Phi_G(\Dx(\cdot;X^\sigma,Y^\sigma))\ge c]\to1$ for some $c>0$.
\end{enumerate}
Then $\Phi_G(\Dx(\cdot;X,Y))\to0$ in probability, $p_\downarrow\to1/(1+B)\le\alpha$, and the unguarded separation test rejects with probability tending to one, although $\U(A;r)\cap\U(B;r)\neq\emptyset$ for every $r\ge r^\ast(A,B)$ in $G$.
\end{theorem}
\begin{proof}
By Proposition~\ref{prop:rate}, off $E_T\cup E_M$ the sample profile equals the population profile on $G$, which is $0$ on $G$: below first contact by Lemma~3.1 of the foundations paper, and beyond it by blindness. Since $|E_T\cap G|=O(\epsilon)$ and $|E_M\cap G|\to0$ by (i), the bound on $\sup_G|\Dx|$ in (i) gives $\Phi_G(\Dx(\cdot;X,Y))\to0$ in probability. By (ii), with probability tending to one every one of the $B$ null replicates exceeds the observed statistic, so $\#\{b:T^{\sigma_b}\le T_{\mathrm{obs}}\}\to0$ and $p_\downarrow\to1/(1+B)$.
\end{proof}

Hypothesis (ii) is the one that must be checked per configuration, and it has two sufficient conditions. Asymptotically, if the pooled support $A\cup B$ itself satisfies the hypotheses of Proposition~\ref{prop:rate} (positive reach fails when $A$ and $B$ cross, but holds for disjoint supports such as two concentric circles) and $\chi(\U(A\cup B;r))\neq0$ on a subset of $G$ of positive measure, then both halves of a random split are dense samples of $A\cup B$ from the mixture density, and the proposition applied to $(A\cup B,A\cup B)$ gives $c=\int_G|\chi(\U(A\cup B;r))|\,dr>0$. At finite sample sizes, (ii) also holds whenever the first-contact scale $r^\ast(A,B)$ exceeds the connectivity scale of the pooled sample, since the two halves of a random split then meet, in many components, at scales where the labeled offsets are still disjoint; this is the regime of Figure~\ref{fig:intro_blindspot}, where $p_\downarrow=0.005$ on two rings whose population profile is zero everywhere. The mechanism the theorem isolates is that $\Dx$ conflates ``empty overlap'' with ``overlap with cancelling Betti numbers'', while the permutation null only ever produces the fat, non-cancelling overlap of a random split.

\begin{lemma}[Consistency of the first-contact scale]\label{lem:rstar}
Let $\hat r^\ast=\tfrac12\min_{i,j}\|x_i-y_j\|$ and suppose the samples are $\epsilon$-dense in $A$ and $B$. Then $r^\ast(A,B)\le\hat r^\ast\le r^\ast(A,B)+\epsilon$.
\end{lemma}
\begin{proof}
The lower bound holds because $X\subset A$ and $Y\subset B$. For the upper bound take $a\in A$, $b\in B$ attaining $r^\ast$ (compactness); by $\epsilon$-density there are $x_i,y_j$ within $\epsilon$ of them, and $\|x_i-y_j\|\le\|a-b\|+2\epsilon$.
\end{proof}

\begin{corollary}[The guard on the blind set]\label{prop:guard}
Let $(A,B)$ be $\chi$-blind on an interval $W=[r^\ast(A,B)+\eta,\,r^\ast(A,B)+\eta']$ and let the guard fire when $\Dx(\cdot;X,Y)$ vanishes on a window $[\hat r^\ast+\eta_0,\hat r^\ast+\eta_1]$ with $\eta\le\eta_0<\eta_1\le\eta'-\epsilon$. Under the sampling of Proposition~\ref{prop:rate} and hypothesis (i) of Theorem~\ref{thm:false} on $W$, the window lies inside $W$ and the guard fires with probability tending to one.
\end{corollary}
\begin{proof}
By Lemma~\ref{lem:rstar}, $\hat r^\ast\in[r^\ast,r^\ast+\epsilon]$ with probability tending to one, so the window lies in $W$, where the population profile is zero; off $E_T\cup E_M$ the sample profile agrees with it, and $|(E_T\cup E_M)\cap W|\to0$.
\end{proof}

The corollary covers the campaign's guard, whose window is the quarter of the grid immediately beyond $\hat r^\ast$, once that quarter lies inside the blind interval; the plateau form of Section~\ref{sec:tests} additionally waits out the contact transient, which the corollary's $\eta_0>0$ models. Neither form has a finite-sample guarantee beyond this mechanism, and Experiment~E4 reports both firing on none of the campaign's cells. $\hat r^\ast$ decides \emph{whether} the clouds interact; $\Dx$ describes \emph{how}; only the second is subject to cancellation. Projection is compatible with the rule $\hat r^\ast>r$ by Proposition~\ref{prop:projection}, which can only lower $\hat r^\ast$; it changes the profile beyond contact, so the silence-based forms are statements about the shadow.

\subsection{The concentration measurements}\label{app:concentration}

Table~\ref{tab:concentration} reports, for the frozen representations of Section~\ref{sec:limitations}, the ambient dimension, the participation-ratio effective rank, the \emph{scale window}---the fractional gap between the first-percentile cross-class distance and the closest cross pair, at $100$ points per class---and, for the two ImageNet backbones on which the projected certificate was attempted, the realized tightness $r^\ast(PX,PY)/r^\ast(X,Y)$ of the $d_0=5$ certificate. For the vision transformers the window is given for the class token and for the mean-pooled patch tokens of the same forward pass. The window divides by a minimum, an extreme order statistic, so its absolute value is specific to the sample cap; ratios between layers of one model are stable across caps ($3.4\times$ at $100$ and $3.6\times$ at $200$ points per class for ResNet-50 stage~1 to stage~4) and are what the text compares. GPT-2's final layer has effective rank near $1$ and its window is not interpretable.

\begin{table}[H]
\centering\small
\begin{tabular}{@{}llrrrrr@{}}
\toprule
model & layer & $D$ & eff.\ rank & window & window (mean-pool) & tightness\\
\midrule
ResNet-50 & stage~1 & 256 & 6 & 0.61 & -- & 0.26\\
 & stage~2 & 512 & 13 & 0.37 & -- & 0.19\\
 & stage~3 & 1024 & 33 & 0.24 & -- & 0.14\\
 & stage~4 & 2048 & 43 & 0.18 & -- & 0.13\\
ViT-B/16 & block~3 & 768 & 9 & 0.44 & 0.47 & 0.22\\
 & block~6 & 768 & 27 & 0.24 & 0.24 & 0.11\\
 & block~9 & 768 & 95 & 0.15 & 0.20 & 0.08\\
 & block~12 & 768 & 46 & 0.14 & 0.28 & 0.10\\
 & final & 768 & 56 & 0.14 & 0.28 & 0.12\\
ViT-L/16 & block~6 & 1024 & 15 & 0.28 & 0.33 & --\\
 & block~12 & 1024 & 91 & 0.13 & 0.22 & --\\
 & block~18 & 1024 & 95 & 0.13 & 0.28 & --\\
 & block~24 & 1024 & 26 & 0.12 & 0.19 & --\\
 & final & 1024 & 81 & 0.10 & 0.12 & --\\
CLIP ViT-B/32 & block~3 & 768 & 7 & 0.44 & -- & --\\
 & block~6 & 768 & 22 & 0.25 & -- & --\\
 & block~9 & 768 & 25 & 0.21 & -- & --\\
 & final & 768 & 40 & 0.23 & -- & --\\
DINOv2 ViT-S/14 & final & 384 & 66 & 0.24 & -- & --\\
BERT-base & block~3 & 768 & 20 & 0.16 & -- & --\\
 & block~6 & 768 & 20 & 0.19 & -- & --\\
 & block~9 & 768 & 25 & 0.20 & -- & --\\
 & final & 768 & 34 & 0.19 & -- & --\\
GPT-2 & block~3 & 768 & 2 & 0.22 & -- & --\\
 & block~6 & 768 & 2 & 0.19 & -- & --\\
 & block~9 & 768 & 6 & 0.22 & -- & --\\
 & final & 768 & 1 & 0.60 & -- & --\\
\bottomrule
\end{tabular}

\caption{Scale window, effective rank, and certificate tightness on frozen representations (CIFAR-10 test images through the vision models; ten classes of 20~Newsgroups through the language models).}
\label{tab:concentration}
\end{table}

\section{Experiment E6 in full: the differentiable surrogate}\label{app:e6}

This appendix gives the E6 campaign of Section~\ref{sec:experiments} in full.

We test whether the surrogate $\mathcal{L}_{\mathrm{dis}}$ of Experiment~E6 can \emph{control} disentanglement, not merely measure it. We train ResNet-20 on CIFAR-10 at the E3 recipe with $\mathcal{L}=\mathcal{L}_{\mathrm{CE}}+\lambda\,\mathcal{L}_{\mathrm{dis}}$ for $\lambda\in\{0,0.5,2,8\}$ ($\lambda=0$ the matched control), $2$ seeds each, and evaluate three quantities per model: clean test accuracy, a robustness proxy (mean accuracy under additive Gaussian input noise at $\sigma\in\{0.05,0.1,0.2\}$), and---closing the loop---the penultimate-layer mean interaction quotient measured \emph{post hoc} by the exact ECP permutation test (the same certified, non-differentiable protocol as E1--E5, $m=200$, $d_0=5$, $B=199$).

\emph{Result (the first campaign steered the wrong way).} With $\sigma^2$ \emph{detached} from the graph---the natural implementation, and the one we ran first---the certified penultimate quotient rose monotonically in $\lambda$: $0.159 \to 0.212 \to 0.278 \to 0.344$ at $\lambda=0, 0.5, 2, 8$, the two seeds tight at every level and all $45$ pairs certified separated in every model, while test accuracy fell $92.0\% \to 86.5\%$ and the robustness proxy $78.9\% \to 72.6\%$. The surrogate made the classes \emph{more} entangled, as measured by the exact test. The cause is instructive. Detaching $\sigma^2$ leaves the loss \emph{value} scale-invariant (verified: identical under every rescaling of the features) but not its \emph{gradient}: backpropagation sees a fixed bandwidth, so the radial component of the gradient is negative and descent inflates feature norms---with the bandwidth frozen, scaling the features by $20$ drives the loss to machine zero. The ECP is genuinely scale-equivariant, so the inflation bought nothing there, and the capacity spent on it appears as the accuracy cost. A scale-invariant loss whose gradient is not scale-invariant is an easy trap, and the certified measurement is what caught it. Keeping $\sigma^2$ in the graph makes the radial gradient vanish identically, so scaling is no longer a descent direction. That closes one exploit; the corrected campaign shows it opens another.

\emph{Result (the corrected campaign steers the wrong way too).} With $\sigma^2$ in the graph, the certified penultimate quotient still rises monotonically in $\lambda$: $0.167 \to 0.250 \to 0.488$ at $\lambda=0, 0.5, 2$ (means of two seeds, which agree within $0.03$ at every level; all $45$ pairs certified separated in every model), while test accuracy falls $92.1\% \to 91.0\% \to 65.7\%$ and the robustness proxy $79.8\% \to 79.7\% \to 62.9\%$. At $\lambda=8$ training fails outright: test accuracy averages $16\%$ ($10\%$ and $22\%$ for the two seeds), one seed's penultimate representation collapsed to a single feature vector (recorded as degenerate; no quotient is defined), and the other sits at quotient $0.88$, statistically indistinguishable from a random relabeling. The second failure has a second mechanism. Once the radial direction is closed, the cheapest way to shrink a within-class-normalized cross-class kernel is to grow the within-class radius: dispersing each class lowers every $\|f(x)-f(y)\|^2/\sigma^2$ without moving the classes apart, and as $\sigma^2$ approaches the cross-class scale the classes interpenetrate, which is exactly what the exact ECP measures. The surrogate's value goes down; the certified entanglement goes up.

\emph{What E6 establishes} is therefore not a regularizer but a measurement result. Two natural implementations of a differentiable surrogate for the interaction quotient---one gamed through scale, one through dispersion---both decrease the surrogate while increasing the certified entanglement, and only the exact, non-differentiable test sees it. A surrogate that provably tracks $\Dx$ under optimization is open, and we would not trust one without the certificate behind it.

\section{Experiment E7 in full: the interaction spectrum}\label{app:e7}

This appendix gives the E7 campaign of Section~\ref{sec:experiments} in full.
On the final-epoch representations of the E2 models and ResNet-20, compute the full spectrum $\widehat E(C_j)$, $j=2,\dots,5$, on three fixed $5$-class subsets ($2^5-1$ pooled Euler curves per evaluation, permutation quotients per floor), the top floors $j\in\{8,9,10\}$ of all ten classes, and the onset vector $r^\ast_2\le\dots\le r^\ast_{10}$ per layer. Deliverables: the decay rate of $\widehat E(C_j)$ in $j$ (the quantitative form of pairwise dominance), whether any deep many-class overlap survives at the top floors, and depth profiles of the onset delays.

\emph{Result: decay.} The spectrum decays monotonically in $j$ in every layer of every model, and the decay rate is itself depth-graded: the mean step ratio $\widehat E(C_{j+1})/\widehat E(C_j)$ over $j=2,\dots,4$ is $\approx 0.88$ at the stem but $\approx 0.61$ at the deepest stage (range $0.40$--$0.83$ across deep layers vs.\ $0.80$--$0.93$ shallow)---pairwise dominance is strongest exactly where disentanglement happens. At $k=10$, the top floors $j\in\{8,9,10\}$ are negligible in the deepest stage of every ResNet (quotients $0.015$--$0.07$, against $0.44$--$0.73$ at the stem) and small in the ViT ($0.09$--$0.23$ at the penultimate block); these ran at $B=19$, where $p_\downarrow=0.05$ is the smallest attainable value, so we report them as magnitude statements rather than sharp certifications.

\emph{Result: the second floor as a layer summary.} The second floor $\widehat E(C_2)$ of a five-class family is one number per layer with one permutation test, and on the three fixed subsets it certifies separation ($p_\downarrow=0.01$ at $B=99$) in every layer of every model except two shallow cells, while ranking the layers exactly as the mean of the ten pairwise quotients does: for the CIFAR-10 ResNet-56, $0.93/0.78/0.90$ at the stem against pairwise means $0.86/0.67/0.82$, falling to $0.41/0.28/0.38$ against $0.30/0.15/0.21$ at stage~3; for the CIFAR-100 subset, $0.09/0.19/0.09$ against $0.04/0.10/0.06$ at stage~3. The floor sits above the pairwise mean throughout, as it must: it is the union of the pairwise overlaps normalized by that union's own null. It also sees the ViT head, rising from $0.71/0.31/0.38$ at block~8 to $0.67/0.60/0.70$ at the penultimate block. A layer-level claim with a single test is therefore available from the same sweep whenever the multiplicity of a pairwise family is unwelcome.

\emph{Result: onsets.} The onset vectors tell the same story in the scale domain. At the CIFAR-10 ResNet-56 penultimate stage the $10$-fold overlap onsets at $r^\ast_{10}\approx1.85$, twenty times the pairwise onset $r^\ast_2\approx0.09$, and the delays stretch monotonically through the vector ($0.09, 0.34, 0.64, 0.91, \dots$); at the stem the entire vector is compressed into $[0.007, 0.069]$. Deep representations do not merely separate pairs---they push every higher-order overlap out to scales far beyond the pairwise interaction range, which is the geometric mechanism behind the dominance of E2, now quantified floor-by-floor (Figure~\ref{fig:e7}).

\begin{figure}[H]
\centering
\includegraphics[width=0.66\textwidth]{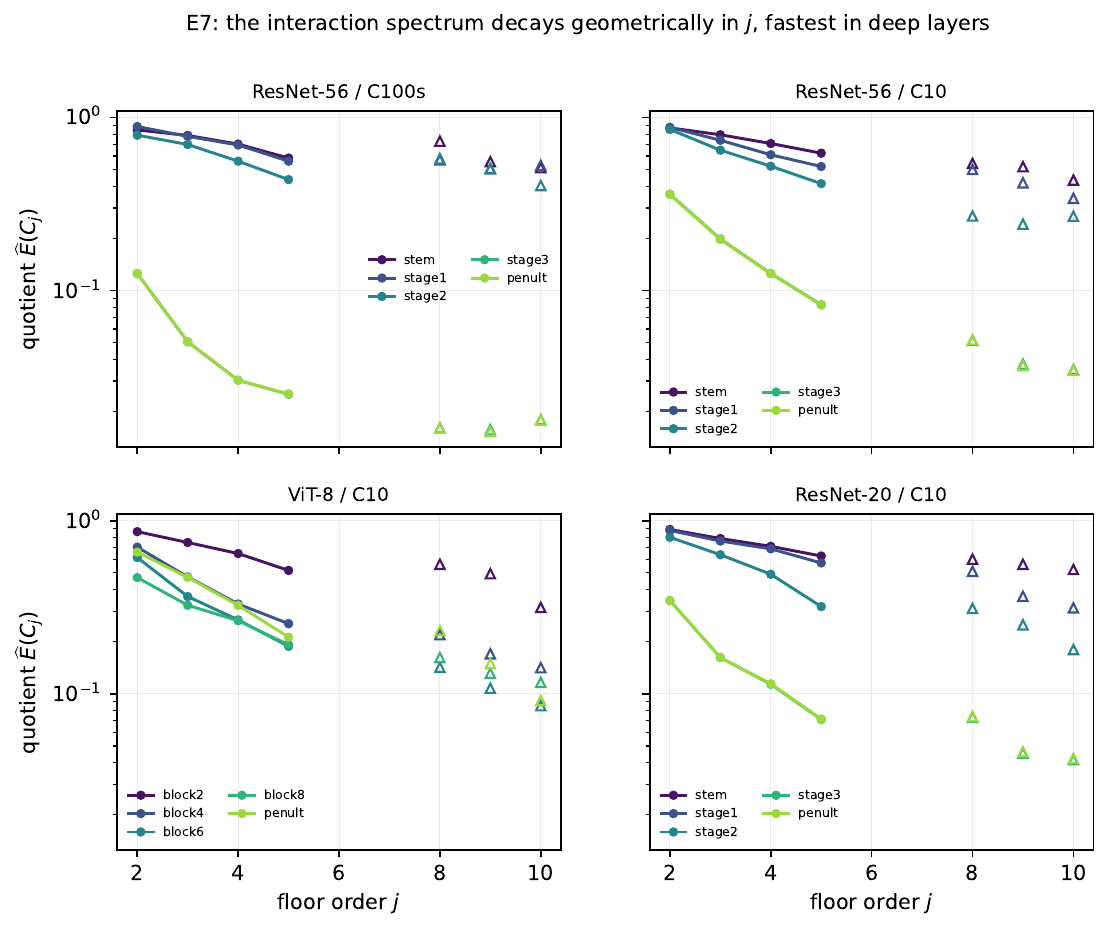}\\[2pt]
\includegraphics[width=0.66\textwidth]{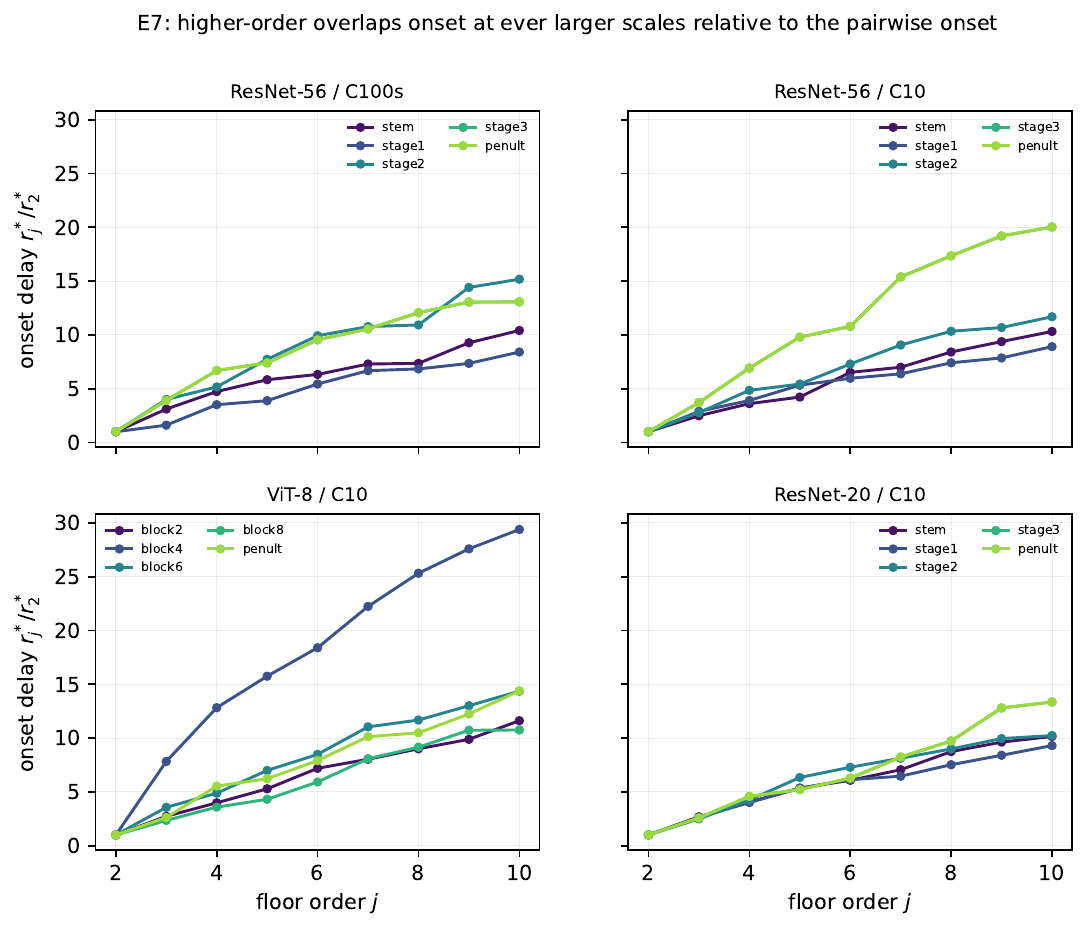}
\caption{The interaction spectrum (E7). \emph{Top:} the floor quotients $\widehat E(C_j)$ decay geometrically in $j$ (solid: $j\le5$ on five-class subsets; open: the top floors of all ten classes, $B=19$), steepest in deep layers. \emph{Bottom:} the onset delays $r^\ast_j/r^\ast_2$ grow with $j$.}
\label{fig:e7}
\end{figure}

\section{Experiment E8 in full: the projection dimension}\label{app:e8}

This appendix gives the E8 campaign of Section~\ref{sec:experiments} in full.
The whole campaign measures on a $d_0=5$ PCA shadow; E1 already sweeps $d_0$ on MNIST, and here we re-measure the depth conclusions of E2 at $d_0\in\{3,4,6\}$ on all three E2 models and the pretrained ViT-B/16, final epochs, full pairwise matrices ($B=199$), comparing each layer's $45$-pair quotient vector against the $d_0=5$ reference.

\emph{Result.} Every qualitative conclusion is invariant to $d_0$. The depth grading is unchanged at every projection dimension---the mean quotient falls monotonically from stem to deep stages in both ResNets (e.g.\ CIFAR-10 ResNet-56 $0.75\to0.22$, $0.72\to0.18$, $0.72\to0.16$ at $d_0=3,4,6$) and the ViT's non-monotone rise into the penultimate block persists at all $d_0$, while the pretrained ViT stays monotone throughout---and deep layers are fully certified separated ($45/45$) at every $d_0$. The pair \emph{ranking} is stable and tightens as $d_0$ grows: Spearman against the $d_0=5$ reference has median $0.91$ ($d_0=3$), $0.96$ ($d_0=4$), $0.97$ ($d_0=6$). The few lower correlations ($\rho=0.56$--$0.69$) occur only at the coarsest $d_0=3$ and only where they cannot change a conclusion: on deep layers whose pairs are all certified separated with near-null quotients (rank noise among indistinguishable-from-zero values) or on the harder CIFAR-100 subset's shallow layers, where the absolute quotients barely move. The measured shadow is low-dimensional, but the disentanglement structure we report is not an artifact of the particular $d_0$.

\section{A coupling we do not claim: grokking}\label{app:grok}

Grokking \citep{power2022grokking} offered a test of whether the measure tracks generalization as it happens: a two-layer transformer on $(a+b)\bmod 97$ at a $40\%$ train fraction, three seeds at weight decay $1.0$ and a non-generalizing control at $0$, checkpointed every $100$ steps across the transition and compared checkpoint to checkpoint with the paired test of Section~\ref{sec:paired}, the fold statistic averaged over the ten residue-class pairs. On the raw profile mass the certified change and the change in validation accuracy were anticorrelated at $r=-0.80$; that was the feature-norm shrinkage that weight decay produces across the transition, the artifact of E3 in another guise. On the scale-free statistic the contemporaneous coupling is $r=-0.33$ ($p=0.03$ over $42$ pooled and autocorrelated intervals; $-0.57$, $-0.34$, $-0.43$ per seed), shuffled class groupings reverse its sign ($+0.42$ to $+0.45$), the control shows none, and no lag is significant. That is the right sign and the right specificity at a magnitude we do not consider a finding, and we report it here rather than in the results.

\bibliographystyle{plainnat}
\bibliography{refs}

\end{document}